\documentclass[journal,twoside,web, twocolumn]{ieeecolor}

\usepackage{generic}
\usepackage[hypertexnames=false]{hyperref}
\hypersetup{
    colorlinks,
    citecolor=red,
    filecolor=black,
    linkcolor=blue,
    urlcolor=black
  }
\usepackage{amsmath}
\usepackage{amsfonts}
\usepackage{amssymb}
\usepackage{algorithmic}
\usepackage{graphicx}
\usepackage{subcaption}
\usepackage{textcomp}
\usepackage{cite}
\usepackage{xspace}
\usepackage{enumerate}
\usepackage{mathtools}
\usepackage{bm}
\usepackage{booktabs}

\newcommand*{\Title}{From gymnastics to virtual nonholonomic constraints: energy
injection, dissipation, and regulation for the acrobot}
\newtheorem{thm}{Theorem}
\newtheorem{prop}{Proposition} 
\newtheorem{lemma}{Lemma} 
\newtheorem{defn}{Definition} 
\newtheorem{assm}{Assumption} 
\newtheorem{rem}{Remark}

\DeclarePairedDelimiter{\abs}{\lvert}{\rvert}
\DeclareMathOperator{\Rank}{rank}
\DeclareMathOperator{\Sign}{sgn}
\DeclareMathOperator{\Diag}{diag}

\newcommand*{\rank}[1]{\Rank\left(#1\right)}
\newcommand*{\sign}[1]{\Sign\left(#1\right)}
\newcommand*{\diag}[1]{\Diag\left(#1\right)}

\newcommand*{\tpose}{^\mathsf{T}} 
\newcommand*{\inv}{^\mathsf{-1}}
\newcommand*{\Rt}[1]{[\R]_{#1}}
\newcommand*{\R}{\mathbb{R}}
\renewcommand*{\Re}{\mathbb{R}}

\newcommand*{\Sone}{\mathbb{S}^1}
\newcommand*{\SxR}{\Sone \times \R}
\newcommand*{\Minv}{M^\mathsf{-1}}
\newcommand*{\Id}[1]{I_{#1}}
\newcommand*{\Zmat}[1]{\bm{0}_{#1}}

\newcommand*{\pdiff}[2]{\frac{\partial #1}{\partial #2}}

\newcommand*{\simpleB}{\begin{bmatrix}\Zmat{(n-k)\times k}\\ \Id{k}\end{bmatrix}}

\newcommand*{\etal}{\MakeLowercase{\textit{et al.~}}}
\newcommand*{\cU}{\mathcal{U}}
\newcommand*{\cV}{\mathcal{V}}

\newcommand*{\cP}{\mathcal{P}}
\newcommand*{\cQ}{\mathcal{Q}}

\newcommand*{\rone}{r}

\newcommand*{\eqdef}{\vcentcolon=}

\newcommand*{\vnhc}{\textsc{vnhc}\xspace}
\newcommand*{\vnhcs}{\textsc{vnhc}s\xspace}

\begin{document}
\title{\Title}
\author{Adan Moran-MacDonald, \IEEEmembership{Member, IEEE}, Manfredi Maggiore,
\IEEEmembership{Senior Member, IEEE}, and Xingbo Wang
\thanks{A. Moran-MacDonald (e-mail: adan.moran@mail.utoronto.ca),
M. Maggiore (e-mail: maggiore@control.utoronto.ca), and
X. Wang (e-mail: xingbo.wang@utoronto.ca) are with the Department of
Electrical and Computer Engineering, University of Toronto, ON, Canada. This research was supported by the Natural Sciences and Engineering Research Council of Canada (NSERC).}
\thanks{This paper was published in the IEEE Transactions on Control Systems Technology, vol. 32, issue 1, pp. 47-60, 2024. DOI: 10.1109/TCST.2023.3294065}
} 

\maketitle

\begin{abstract}
    In this article we study virtual nonholonomic constraints, which are
    relations between the generalized coordinates and momenta of a
    mechanical system that can be enforced via feedback control.
    We design a constraint which emulates gymnastics giant motion in an
    acrobot, and prove that this constraint can inject or dissipate
    energy based on the sign of a design parameter.
    The proposed constraint is tested both in simulation and experimentally on a real-world acrobot,
    demonstrating highly effective energy regulation properties and robustness
    to a variety of disturbances.
\end{abstract}


\section{Introduction}\label{sec:introduction}

In gymnastics terminology, a ``giant" is the motion a gymnast performs to
achieve full rotations around a horizontal bar \cite{usagym_giant}. 
A gymnast will begin by hanging at rest, then swing their legs
appropriately to gain energy over time.
The authors of \cite{pendulum_length_giant_gymnastics} modelled the gymnast as a
variable length pendulum, and studied how the pendulum's length changes as a
function of the gymnast's limb angle.
Labelling the pendulum length by \(r\) and the gymnast's body orientation
by \(\theta\), they observed experimentally that the value \(\dot{r}/r\) has
the biggest impact on the magnitude of energy injection. 
After testing several gymnasts under a variety of experimental conditions, 
they discovered that the peak value of \(\dot{r}/r\) occurred at the same fixed
value of \(\dot{\theta}/\theta\) for all gymnasts.
In other words, gymnasts appear to move their legs as a function of their body
angle and velocity when performing giants; 
doing so allows them to gain energy and rotate around the bar.

While the simplest model of a gymnast is the variable-length pendulum, a
more realistic model is the two-link acrobot (Figure \ref{fig:acrobot}).
Here, the top link represents the torso while the bottom link represents
the legs. 
The acrobot is actuated exclusively at the centre joint (the hip).
To solve the swingup problem, one might begin by designing a leg controller
which provably injects energy into the acrobot, so that the resulting motion
mimics that of a human performing a giant.

\begin{figure}
    \centering
    \begin{subfigure}[t]{0.40\linewidth}
        \includegraphics[width=\linewidth]{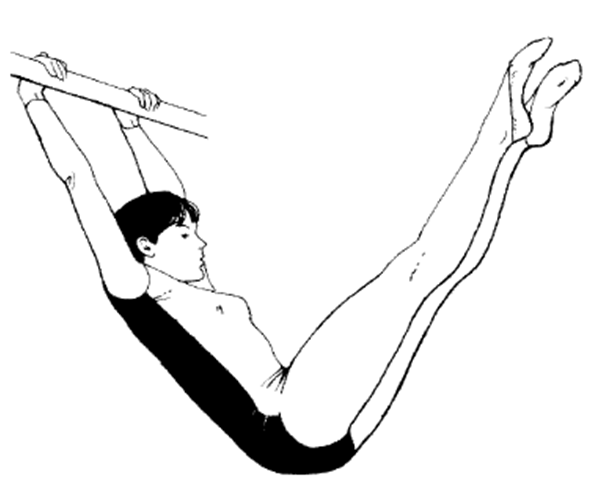}
    \end{subfigure}
    \hfill
    \begin{subfigure}[t]{0.48\linewidth}
        \includegraphics[width=\linewidth]{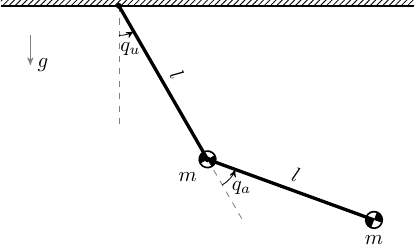}
    \end{subfigure}
    \caption{A simplified two-link acrobot as a model for a gymnast.
    Image modified from \cite{xingbo_thesis}.}
    \label{fig:acrobot}
\end{figure}

Previous attempts at acrobot giant generation have involved
trajectory tracking, partial feedback linearization, or other energy-based
methods
(see
\cite{energy_pumping_robotic_swinging,swingup_giant_acrobot,dynamical_servo_acrobot_vc,control_giant_two_link_gymnastic_robot}
).
While all these approaches succeed at making the acrobot rotate around the
bar, none of them use the results of \cite{pendulum_length_giant_gymnastics}.
That is, none of these leg controllers track a function of the acrobot's body
angle and velocity.
In 2016, Wang designed a controller which tracked the
body angle and an \textit{estimate} of the velocity, but not the velocity itself
\cite{xingbo_thesis}.
His approach was a preliminary version of a recent technique know as the method
of virtual nonholonomic constraints.

Virtual nonholonomic constraints (\vnhcs) have been used for human-robot interaction
\cite{vnhc_human_robot_cooperation,psd_based_vnhc_redundant_manipulator,haptic_vnhc},
error-reduction on time-delayed systems \cite{vnhc_time_delay_teleop},
and they have shown marked improvements to the field of bipedal locomotion 
\cite{nhvc_dynamic_walking,
hybrid_zero_dynamics_bipedal_nhvcs,output_nhvc_bipedal_control}.
Indeed, they produce more robust walking motion in biped robots than
other virtual constraints which do not depend on velocity
\cite{nhvc_incline_walking}.
In particular, \vnhcs may be capable of injecting and
dissipating energy from a system in a robust manner, all while producing
realistic biological motion. The recent paper~\cite{SimStraBloCol23} initiates a systematic geometric investigation of \vnhcs that are affine functions of the configuration velocities.

\textbf{Contributions of this article:}
In this article, we design a \vnhc which provably
injects or dissipates energy into the acrobot depending on the sign of a design parameter. The proposed constraint produces a human-like giant motion.
The main result of the paper is presented in Theorems~\ref{thm:acrobot-oscillations} and~\ref{thm:acrobot-rotations}.   Theorem~\ref{thm:acrobot-oscillations} shows that once the virtual constraint has been enforced, for any nonzero ``low-energy'' initial condition, the proposed controller injects energy in the acrobot making it exit, in finite time, any compact subset of a certain region of the state space where the acrobot oscillates about its pivot point without performing full rotations. Theorem~\ref{thm:acrobot-rotations} handles the case when the acrobot is initialized in the complement of the above region, and it shows that under a mild assumption on the physical parameters, the acrobot  is guaranteed to keep increasing its energy while performing full rotations about the pivot point. By changing the sign of the design parameter, the results of Theorems~\ref{thm:acrobot-oscillations} and~\ref{thm:acrobot-rotations} hold for energy dissipation. By combining these energy injection/damping properties, we then propose an elementary switching logic to attain approximate energy regulation. The results of this paper are demonstrated both in simulation and experimentally. The experimental results, in particular, demonstrate the robustness of the controller against model uncertainty, sensor noise, and a variety of external
disturbances. 

\textbf{Related literature:} While the focus of this paper is on energy injection and dissipation, a closely related concept is that of orbital stabilization to a particular energy level set. The  problem of orbital stabilization has been investigated extensively and we mention four main approaches.

  One approach, applicable chiefly to underactuated mechanical systems with degree of underactuation one, relies on virtual holonomic constraints to create a target closed orbit and some mechanism to stabilize this orbit. In~\cite{ShiPerWit05}, the mechanical system is linearized around the orbit and a linear time-varying controller is designed to stabilize it. In~\cite{dynamic_vhcs_stabilize_closed_orbits}, the constraint is embedded in a one-parameter family, and a dynamic controller is designed to adapt the parameter so as to stabilize the target orbit.

  A second approach to orbital stabilization, developed in~\cite{OrtYiRomAst20} and applicable to general control-affine systems, is an adaptation of the immersion and invariance (I\&I) technique of~\cite{AstOrt03}. In this approach, one seeks a controlled invariant submanifold of the state space on which the system has desired oscillator dynamics, and designs a feedback controller to stabilize this submanifold. The result is a controller that stabilizes \emph{some} closed orbit on the manifold, depending on the initial condition. This is in contrast to the approaches in~\cite{ShiPerWit05,dynamic_vhcs_stabilize_closed_orbits} that are able to fix the closed orbit to be stabilized.

  A third approach, developed in~\cite{YiOrtWuZha20}, is an adaptation of the passivity-based interconnection and damping assignment (IDA) technique of~\cite{OrtSpoGomBlan02} and is applicable to underactuated port-Hamiltonian systems. In~\cite{YiOrtWuZha20}, the authors assume that the target closed orbit is contained in a two-dimensional plane and present two ways of stabilizing it. The first way relies on a modification of energy shaping via IDA in which the component of the shaped energy on the plane has a minimum at the target closed orbit so that its graph resembles a Mexican sombrero.  The second way represents the closed orbit as an energy level set and assigns the damping matrix in such a way that the controller pumps  energy when the energy is below the target level and damps energy otherwise.

  A final approach, developed in~\cite{FanLoz02}, stabilizes homoclinic orbits of pendulum-like systems using energy methods and Lyapunov-based passivity techniques. The stabilization of a homoclinic orbit allows the designer to inject energy in the mechanical system, and it typically yields large domains of attraction.
  
 We now place the contributions of this paper in the context of the literature just reviewed. It is important to bear in mind that the ideas put forth here are limited in scope to the acrobot. They are nonetheless notable in that they explore a mechanism to pump or dissipate energy  that is guaranteed to work in a large region of the state space. Theorems~\ref{thm:acrobot-oscillations} and~\ref{thm:acrobot-rotations} make this precise. In contrast, many the approaches found in the literature and reviewed  above yield only local asymptotic stabilization of a closed orbit. When it comes to the acrobot example, the employment of \vnhcs allows one to incorporate knowledge about giant maneuvers in gymnastic into the control design, something that cannot be easily achieved with existing techniques.

 There are similarities between the virtual constraint approach and the I\&I technique for orbital stabilization  proposed in~\cite{OrtYiRomAst20} in that both approaches seek a controlled invariant submanifold on which the closed-loop system exhibits desired dynamics. But while in~\cite{OrtYiRomAst20} one wishes to obtain oscillator dynamics on the submanifold, in this paper we wish to obtain expansive or contractive dynamics (depending on the sign of a design parameter) which rules out oscillator dynamics.

In this paper, like in~\cite{YiOrtWuZha20}, we design a controller for energy pumping/damping, but the energy pumping/damping mechanism proposed here does not arise from any form of energy shaping. In particular, it does not arise from an assignment of the damping matrix of the closed-loop system.

Finally, while in~\cite{FanLoz02} the emphasis is typically the stabilization of a specific energy level set corresponding to the homoclinic orbit of an unstable equilibrium, the technique used in this paper to inject energy does not involve the stabilization of any homoclinic orbit, and as such it does not limit the energy level that can be attained.

\textbf{Paper organization:}
In Section~\ref{sec:problem-formulation} we review the acrobot model and formulate the problem investigated in this paper. Section~\ref{sec:vnhc} presents a class of simply actuated Hamiltonian systems and preliminary notions on \vnhcs. In particular, Theorem~\ref{thm:vnhc-regularity} in this section characterizes the constrained dynamics arising from a class of \vnhcs. This result is then compared with existing literature on the subject. Section~\ref{sec:acrobot} presents the proposed \vnhc for the acrobot and the two main results of the paper, Theorems~\ref{thm:acrobot-oscillations} and~\ref{thm:acrobot-rotations} mentioned earlier. Sections~\ref{sec:simulations} and~\ref{sec:experiments} present simulations and experimental results. Section~\ref{sec:proof} presents the proof of Theorems~\ref{thm:acrobot-oscillations} and~\ref{thm:acrobot-rotations}, and Section~\ref{sec:conclusion} presents concluding remarks.

\textbf{Notation}:
We use the following notation and terminology in this article.
The \(n \times n\) identity matrix is denoted \(\Id{n}\), and the \(n \times m\)
matrix of zeros is denoted \(\Zmat{n\times m}\).
A matrix \(A \in \R^{n \times m}\) is \textit{right semi-orthogonal} if
\(A A\tpose = \Id{n}\) and is \textit{left semi-orthogonal} if 
\(A\tpose A = \Id{m}\).
For \(A \in \R^{n\times m}\) and \(B \in \R^{p \times m}\),
we define \([A;B] \in \R^{(n+p)\times m}\) as the matrix obtained by stacking \(A\)
on top of \(B\). 
Given \(\sigma_1,\ldots,\sigma_n \in \R\), we define 
\(\diag{\sigma_1,\ldots,\sigma_n} \in \R^{n \times n}\) as the diagonal matrix
whose value at row \(i\), column \(i\) is \(\sigma_i\).
For \(T > 0\), the set of real numbers modulo \(T\) is denoted \(\Rt{T}\), with
\(\Rt{\infty} \eqdef \R\). 
The gradient of a matrix-valued function 
\(A : \R^m \rightarrow \R^{n\times n}\) is the block matrix of stacked partial
derivatives, 
\(\nabla_xA \eqdef [\partial A/\partial x_1; \cdots; \partial A/\partial x_m] \in
\R^{nm \times n}\).
Given two matrices \(A \in \R^{n \times m}\) and \(B \in \R^{r \times s}\), the
Kronecker product (see \cite{kronprod}) is the matrix  
\(A \otimes B \in \R^{nr \times ms}\)  defined as
\begin{equation}\label{eqn:kronprod}
    A \otimes B = \begin{bmatrix}
        A_{1,1}B & \cdots & A_{1,m} B \\
        \vdots & \ddots & \vdots \\
        A_{n,1} B & \cdots & A_{n,m} B
    \end{bmatrix} 
    .
\end{equation}
The Poisson bracket \cite{landau_mechanics} between the functions
\(f(q,p)\) and \(g(q,p)\) is
\begin{equation}\label{eqn:poisson-bracket}
    [f,g] \eqdef \sum \limits_{i=1}^n \pdiff{f}{p_i}\pdiff{g}{q_i} - 
        \pdiff{f}{q_i}\pdiff{g}{p_i}
    .
\end{equation}
The Kronecker delta \(\delta_i^j\) is \(1\) if \(i = j\) and \(0\)
otherwise.
Finally, we say a function \(\Delta(I)\) is \(O(I^2)\)
if \(\lim_{I \to 0} \Delta(I)/I = 0\).

\section{Problem Formulation}\label{sec:problem-formulation}
We will use the simplified acrobot model in Figure \ref{fig:acrobot}, where we assume the torso and leg rods are of equal length \(l\) with equal point masses \(m\) at the tips. The acrobot's configuration is described in generalized coordinates \(q=(q_u,q_a)\) on the configuration manifold \(\mathcal{Q} = \Sone \times \Sone\), where \(q_u\) is unactuated and  \(q_a\) is actuated. We ignore dissipative forces in this model.

In what follows, we denote by \(c_u \eqdef \cos(q_u)\), \(c_a \eqdef \cos(q_a)\), and \(c_{ua} \eqdef \cos(q_u + q_a)\); likewise, \(s_u \eqdef \sin(q_u)\), \(s_a \eqdef \sin(q_a)\), and \(s_{ua} \eqdef \sin(q_u + q_a)\).

The acrobot has inertia matrix \(M\), potential function \(V\) (with respect to the horizontal bar), and input matrix \(B\) given as follows:
\begin{align}\label{eqn:acrobot-inertia}
    M(q) &= \begin{bmatrix}
        ml^2\left(3+2 c_a\right) & 
        ml^2\left(1+ c_a \right) \\
        ml^2\left(1+ c_a\right) &
        ml^2
    \end{bmatrix} 
    , \\
    \label{eqn:acrobot-potential}
    V(q) &= -mgl\left(2 c_u + c_{ua}\right)
    , \\
    \label{eqn:acrobot-B}
    B &= [0;1]
    .
\end{align}
For reasons that will become clear later in this article, we
use Hamiltonian mechanics to derive the dynamics of the acrobot. 
To this end, recalling the partitioning  \(q= (q_u,q_a)\), we define the conjugate of momenta, \(p = (p_u,p_a) = M(q)\dot{q}\) and the Hamiltonian function
\begin{equation}\label{eqn:acrobot-hamiltonian}
    \mathcal{H}(q,p) = \frac{1}{2}p\tpose \Minv(q) p -
   mgl\left(2 c_u + c_{ua}\right).
  \end{equation}
The dynamics of the
acrobot with Hamiltonian function~\eqref{eqn:acrobot-hamiltonian} are
\begin{equation}
  \label{eq:acrobot:model}
  \begin{aligned}
&    \dot{q} = \Minv(q) p,\\
&    \dot{p}_u = -mgl\left(2s_u + s_{ua}\right),\\
&     \dot{p}_a =-\frac{1}{2}p\tpose \nabla_{q_a}\Minv(q) p - m g l s_{ua} + \tau,
  \end{aligned}
\end{equation}
where \(\tau\), the control input, is the hip joint torque affecting only the dynamics of
\(p_a\).

Our goal is to design a smooth function \(f : \Sone \times \R \to \Sone\) such that the
relation \(q_a = f(q_u,p_u)\) for system \eqref{eqn:acrobot-hamiltonian} can be
enforced asymptotically via feedback control (in Section \ref{sec:vnhc} we call
this a  virtual nonholonomic constraint).   We will further require that the
dynamics of the acrobot, when the relation holds, gain or lose energy in a sense
that will be defined precisely in Section~\ref{sec:energy-inject-diss}.

\section{Preliminaries on \vnhcs}\label{sec:vnhc}
Before embarking on the design problem, we must summarize the relevant theory of \vnhcs for a class of mechanical systems we call ``simply actuated Hamiltonian systems". Although the results in Section \ref{sec:vnhc-vnhc} do not appear in the literature in the same form, we will show that they are a special case of results that appeared in \cite{hybrid_zero_dynamics_bipedal_nhvcs}.

\subsection{Simply Actuated Hamiltonian Systems} \label{sec:vnhc-sah}
Consider a mechanical system modelled with generalized coordinates  \(q = (q_1, \ldots, q_n)\) on a configuration manifold \(\mathcal{Q} = \Rt{T_1} \times \cdots \times \Rt{T_n}\), where \(T_i = 2\pi\) if \(q_i\) is an angle and \(T_i = \infty\) if \(q_i\) is a displacement. The corresponding generalized velocities are \(\dot{q} = (\dot{q}_1,\ldots,\dot{q}_n) \in \R^n\).

Suppose this system has Lagrangian \(\mathcal{L}(q,\dot{q}) = 1/2~\dot{q}^T D(q) \dot{q} - P(q)\), where the potential energy  \(P : \mathcal{Q} \rightarrow \mathbb{R}\)  is smooth, and the inertia matrix  \(D : \mathcal{Q} \rightarrow \mathbb{R}^{n \times n}\) is smooth and positive definite for all \(q \in \mathcal{Q}\). The \textit{conjugate momentum} to \(q\) is the vector \(p \eqdef \partial\mathcal{L}/\partial\dot{q} = D(q) \dot{q} \in \R^n\). As per \cite{landau_mechanics},  the \textit{Hamiltonian} of the system in \((q,p)\) coordinates is
\begin{equation}\label{eqn:hamiltonian}
    \mathcal{H}(q,p) = \frac{1}{2} p\tpose D\inv(q) p + P(q)
    ,
\end{equation}
with dynamics
\begin{equation}\label{eqn:hamiltonian-eom-general}
    \begin{cases}
        \dot{q} = \nabla_p\mathcal{H} 
        , \\
        \dot{p} = -\nabla_q\mathcal{H} + B \, \tau
        ,
    \end{cases}
\end{equation}
where \(\tau \in \R^k\) is a vector of generalized input forces and the input matrix \(B \in \R^{n \times k}\) is assumed to be constant and of full rank \(k\). If \(k < n\), we say the system is \textit{underactuated} with degree of underactuation \((n-k)\).

Using the matrix Kronecker product, it is easy to show that
\eqref{eqn:hamiltonian-eom-general} expands to
\begin{equation*}\label{eqn:hamiltonian-full-dynamics}
     \begin{cases}
        \dot{q} = D\inv(q)p 
        , \\
        \dot{p} = -\frac{1}{2} (\Id{n} \otimes p\tpose) \nabla_q D\inv(q) p
        - \nabla_q P(q) + B \,\tau
        . \\
    \end{cases}
\end{equation*}

Because \(\tau\) is transformed by \(B\), it is not obvious how any particular input force \(\tau_i\) affects the system.

When \(\mathcal{Q} = \R^n\), one may define a canonical coordinate transformation of \eqref{eqn:hamiltonian} which decouples the input forces. To define this transformation we will make use of the following lemma\footnote{Lemma~\ref{lemma:B-orthogonal} and Theorem~\ref{thm:simply-actuated} are obvious, but we include their proofs for completeness.}.

\begin{lemma}\label{lemma:B-orthogonal}
There exists a nonsingular matrix \(\hat{T} \in \R^{k \times k}\) 
    so that the regular feedback transformation 
    \[
        \tau = \hat{T} \hat{\tau}
    \] 
    has a new input matrix \(\hat{B}\) for \(\hat{\tau}\) which is left
    semi-orthogonal.  
\end{lemma}
\begin{proof}
    Since \(B\) is constant and full rank, it has a singular value decomposition 
    \(B = U\tpose \Sigma V\) where 
    \(\Sigma = [\diag{\sigma_1,\ldots,\sigma_k}; \Zmat{(n-k)\times k}]\),
    \(\sigma_i > 0\), and \(U \in R^{n \times n}\),
    \(V \in \R^{k \times k}\) are unitary matrices \cite{calculating_svd}.
    Defining \(T = \diag{1/\sigma_1,\ldots,1/\sigma_k}\) and assigning the
    regular feedback transformation \(\tau = V T \hat{\tau}\) yields a new input
    matrix \(\hat{B} = B V T\) for \(\hat{\tau}\) such that
    \(\hat{B}\tpose \hat{B} = T\tpose \Sigma\tpose \Sigma T = \Id{k}\).
\end{proof}

In light of Lemma \ref{lemma:B-orthogonal}, there is no loss of generality in
assuming that the input matrix \(B\) is left semi-orthogonal, which means that
\(B\tpose\) is right semi-orthogonal.
Now, let \(\mathbf{B} \eqdef [B^\perp; B\tpose]\) where 
\(B^\perp \in \R^{(n-k)\times k}\) is a full-rank left annihilator of \(B\), i.e.,
\(B^\perp B = \Zmat{(n-k) \times k}\).
Since \(B\) is constant, such a \(B^\perp\) exists and
\(\mathbf{B}\) is invertible.

The following theorem requires that \(\mathcal{Q} = \R^n\) so that the
coordinate transformation is well defined.

\begin{thm}\label{thm:simply-actuated}
For the Hamiltonian system \eqref{eqn:hamiltonian} with configuration manifold \(\mathcal{Q} = \R^n\), the coordinate transformation \((\tilde q,\tilde p) = \left(\mathbf{B}\tpose)\inv q,\mathbf{B}p\right)\) is a canonical transformation. The resulting Hamiltonian function is 
    \begin{equation}
    \label{eqn:simple-hamiltonian}
        \mathcal{H}(\tilde{q},\tilde{p}) = 
        \frac{1}{2} \tilde{p}\tpose \Minv(\tilde{q}) \tilde{p} + V(\tilde{q}),
      \end{equation}
      and the associated Hamiltonian dynamics are
      \begin{equation}
        \label{eq:Hamiltonian_system:simply_actuated}
          \begin{aligned}
            &           \dot{\tilde{q}} = \Minv(\tilde{q})\tilde{p}, \\
            &           \dot{\tilde{p}} = -\frac{1}{2} (\Id{n} \otimes \tilde{p}\tpose)
              \nabla_{\tilde{q}} \Minv(\tilde{q}) \tilde{p} \\
            &          \phantom{---} - \nabla_{\tilde{q}} V(\tilde{q}) + \simpleB \tau,
          \end{aligned}
  \end{equation}
  where 
    \(\Minv(\tilde{q}) \eqdef 
    (\mathbf{B}\tpose)\inv D^{-1}(\mathbf{B}\tpose \tilde{q})\mathbf{B}\inv\)
    and
    \(V(\tilde{q}) \eqdef P(\mathbf{B}\tpose \tilde{q})\).
\end{thm}
\begin{proof}
    Since \(\mathbf{B}\) is constant, this transformation satisfies
    \(\partial\tilde{q}_i/\partial p_j = \partial\tilde{p}_i/\partial q_j = 0\) for all 
    \(i,j \in \{1,\ldots,n\}\).
    This implies the Poisson brackets \([\tilde{q}_i, \tilde{q}_j]\)
    and \([\tilde{p}_i,\tilde{p}_j]\) are both zero.
    Then, one can show that
    \([\tilde{p}_i, \tilde{q}_j] = (\mathbf{B}\tpose)\inv_i (\mathbf{B}\tpose)_j
        = \delta_i^j\).
    By (45.10) in \cite{landau_mechanics}, this is a canonical transformation
    and the new Hamiltonian is
    \(\mathcal{H}(\mathbf{B}\tpose \tilde{q}, \mathbf{B}\inv \tilde{p})\).
    Finally, since \(\dot{\tilde{p}} = \mathbf{B} \dot{p}\), the input
    matrix for the system in \((\tilde{q},\tilde{p})\) coordinates is
    \(\mathbf{B}B = [\Zmat{(n-k)\times k}; \Id{k}]\), which proves the theorem.
\end{proof}

We call the \((\tilde{q},\tilde{p})\) coordinates \textit{simply actuated coordinates}, and we call any Hamiltonian system of the form~\eqref{eq:Hamiltonian_system:simply_actuated} a  \textit{simply actuated Hamiltonian system}. The first \((n-k)\) configuration variables in \(\tilde{q}\), labelled \(q_u\), are the \textit{unactuated coordinates};  the remaining \(k\) configuration variables, labelled \(q_a\), are the \textit{actuated coordinates}. The corresponding \((p_u, p_a)\) in \(\tilde{p}\) are the \textit{unactuated} and \textit{actuated momenta}, respectively.

In this article we focus on the acrobot which, despite having a configuration
manifold which is not \(\R^n\), is already a simply actuated Hamiltonian system.

\subsection{Virtual Nonholonomic Constraints}\label{sec:vnhc-vnhc}

Griffin and Grizzle \cite{nhvc_dynamic_walking} were the first to define relative degree two nonholonomic constraints which can be enforced through state feedback. Horn \etal extended their results in \cite{hybrid_zero_dynamics_bipedal_nhvcs} to derive the constrained dynamics for a certain class of mechanical systems.  Hamed and Ames \cite{nonholonomic_hybrid_zero_dynamics} then found an explicit representation of these constrained dynamics for \vnhcs with a specific form. All these researchers made use of the unactuated conjugate momentum, but they developed their results in the Lagrangian framework. In particular, they focused on Lagrangian systems with degree of underactuation one. The recent paper~\cite{SimStraBloCol23} has initiated a systematic geometric investigation of \vnhcs that are affine functions of the configuration velocities. The ones used in this paper are nonlinear functions of the configuration velocities.

We will now present a special case of
\cite{hybrid_zero_dynamics_bipedal_nhvcs} for simply actuated Hamiltonian
systems, so that the theory we apply to the acrobot is provided in its clearest
form. After we finish this summary we will clarify the relationship between our
material and that of 
\cite{hybrid_zero_dynamics_bipedal_nhvcs,nonholonomic_hybrid_zero_dynamics}.

The rest of this section refers to the simply actuated system
\eqref{eqn:simple-hamiltonian}. For simplicity of notation, we relabel
\((\tilde{q},\tilde{p})\) to \((q,p)\). We assume henceforth that 
$(q,p) \in \cQ \times \Re^n$.

\begin{defn}\label{defn:vnhc}
    A \textit{virtual nonholonomic constraint} (\vnhc) \textit{of order \(k\)}
    is a relation \(h(q,p) = 0\) where 
    \(h : \mathcal{Q}\times\R^n \rightarrow \R^k\) is \(C^2\),
    \(\rank{\left[ dh_q,\, dh_p \right]} = k\) for all 
    \((q,p) \in h\inv(0)\), and there exists a feedback
    controller \(\tau(q,p)\) rendering the \textit{constraint manifold}
    \(\Gamma\) invariant, where
\begin{equation}
        \Gamma = \left\{(q,p) \mid h(q,p) = 0, dh_q \dot{q} + dh_p \dot{p} = 0\right\}.    \end{equation}
\end{defn}

The constraint manifold is a \(2(n-k)\)-dimensional closed embedded submanifold
of \(\mathcal{Q} \times \R^n\). A \vnhc thereby allows us to study a
reduced-order model of the system: it reduces the original \(2n\) differential
equations to \(2(n-k)\) equations. In particular, if the mechanical system has
degree of underactuation one, i.e., \(k = (n-1)\), the constraint manifold is
\textit{always} two-dimensional.

In order to enforce the constraint \(h(q,p) = 0\), we want to asymptotically
stabilize the set \(\Gamma\). To see when this is possible, let us define the
error output \(e = h(q,p)\). If any component of \(e_i\) has relative degree 1,
we may not be able to stabilize \(\Gamma\): we can always guarantee 
\(e_i \to 0\), but not necessarily \(\dot{e}_i \to 0\). It is for this reason
that we define the following special type of \vnhc.

\begin{defn}
A \vnhc \(h(q,p) = 0\) of order \(k\) is \textit{regular} if system~\eqref{eq:Hamiltonian_system:simply_actuated} with output \(e = h(q,p)\) has vector relative degree \(\{2,2.\ldots,2\}\) everywhere on the constraint manifold \(\Gamma\).
\end{defn}

The reader is referred to~\cite{Isi95} for the definition of vector relative degree. The authors of \cite{nhvc_dynamic_walking,hybrid_zero_dynamics_bipedal_nhvcs} observed that relations which use only the unactuated conjugate momentum often have vector relative degree \(\{2,\ldots,2\}\). Indeed, we shall now prove that regular constraints cannot depend on the actuated momentum.

To ease notation in the rest of this section, we use the following shorthand:
\begin{align}
    \mathcal{A}(q,p_u) &\eqdef dh_q(q,p_u) \Minv(q) 
        ,\\
    \mathcal{M}(q,p) &\eqdef (\Id{n-k} \otimes p\tpose)\nabla_{q_u}\Minv(q) 
    .
\end{align}

\begin{thm}\label{thm:vnhc-regularity}
    A relation \(h(q,p) = 0\) for system \eqref{eqn:simple-hamiltonian}
    is a regular \vnhc of order \(k\) if and only if 
    \(dh_{p_a} = \Zmat{k \times k}\) 
    and the decoupling matrix
    \begin{equation}\label{eqn:decoupling-matrix}
        H(q,p) \eqdef \big(\mathcal{A}(q,p_u) - dh_{p_u}\mathcal{M}(q,p)\big)\simpleB
         ,
     \end{equation}
    is invertible everywhere on the constraint manifold \(\Gamma\).
\end{thm}
\begin{proof}
    Let \(e = h(q,p) \in \R^k\).
    If \(dh_{p_a} \neq \Zmat{k\times k}\) for some \((q,p) \in \Gamma\), 
    then \(\tau\) appears in \(\dot{e}\) and the \vnhc is not of relative degree
    \(\{2,\ldots,2\}\). Suppose now that \(dh_{p_a} = \Zmat{k\times k}\).
    Then 
    \(\dot{e} = \mathcal{A}(q,p_u)p - 
     dh_{p_u}\left(1/2~\mathcal{M}(q,p)p + \nabla_{q_u}V(q)\right)\).
    Taking one further derivative provides
    \( \ddot{e} = (\star) - 
        dh_{p_u}\left(1/2~d/dt\left(\mathcal{M}(q,p)p\right)\right) 
        + \mathcal{A}(q,p_u)[\Zmat{(n-k)\times k};\Id{k}] \tau\),
    where \((\star)\) is a continuous function of \(q\) and \(p\).
    One can further show that
    \(dh_{p_u}\left(1/2~d/dt~\left(\mathcal{M}(q,p)p\right)\right)
        = (\star) + dh_{p_u}\mathcal{M}(q,p)[\Zmat{(n-k)\times k};
        \Id{k}]\tau\).
    Hence, $\ddot e$ has the form \( \ddot{e} = E(q,p) + H(q,p)\tau\) for appropriate \(E\).
    From the definition of regularity, the \vnhc \(h\) is regular 
    when \(e\) is of relative degree \(\{2,\ldots,2\}\), which is true 
    if and only if the matrix $H(q,p)$ premultiplying \(\tau\) is nonsingular. This proves the theorem.
\end{proof}

Under additional mild conditions (see \cite{vhcs_for_el_systems}), the constraint manifold associated with a regular \vnhc of order \(k\) can be asymptotically stabilized by the input-output feedback linearizing controller
\begin{equation}\label{eqn:stabilizing-controller}
    \tau(q,p) = -H\inv(q,p)\left(E(q,p) + k_p e + k_d \dot{e}\right)
    ,
\end{equation}
where \(k_p, k_d > 0\) are control parameters which can be tuned on the
resulting error dynamics \(\ddot{e} = -k_p e - k_d \dot{e}\).

In Section \ref{sec:acrobot} we will enforce a regular constraint on the
acrobot of the form \(h(q,p) = q_a - f(q_u,p_u)\), where the actuators track a
function of the unactuated variables.
Regular constraints of this form always meet the mild conditions from
\cite{vhcs_for_el_systems}, and hence we can stabilize the constraint manifold
using \eqref{eqn:stabilizing-controller}.
Since \(q_a\) is constrained to be a function of the unactuated variables,
intuition tells us the constrained dynamics should be parameterized by 
\((q_u, p_u)\).
Unfortunately, \(\dot{q}_u\) depends on \(p_a\), and for general systems one
cannot solve explicitly for \(p_a\) in terms of \((q_u,p_u)\) because
the \(\dot{p}\) dynamics contain the coupling term 
\((\Id{n} \otimes p\tpose)\nabla_{q}M(q)p\). 
We now introduce an assumption which allows us to solve for \(p_a\) as a
function of \((q_u,p_u)\), which in turn allows us to explicitly solve 
for the constrained dynamics.

\begin{assm}\label{assm:inertially-actuated}
The inertia matrix does not depend on the unactuated coordinates, i.e.,
\(\nabla_{q_u}M(q) = \Zmat{n(n-k) \times n}\).
\end{assm}

\begin{thm}\label{thm:zero-dynamics}
    Let \(\mathcal{H}\) be a Hamiltonian system in simply actuated form
    \eqref{eqn:simple-hamiltonian} satisfying
    Assumption \ref{assm:inertially-actuated}. 
    Let \(h(q,p_u) = q_a - f(q_u,p_u)\) be a regular \vnhc of order \(k\) with
    constraint manifold \(\Gamma\). 
    Then the constrained dynamics are given by
    \begin{equation}\label{eqn:qpu-dynamics}
        \left.\begin{aligned}
                \dot{q}_u &= 
                \left[\Id{(n-k)} ~ \Zmat{(n-k) \times k}\right]
                \Minv(q)p \\
            \dot{p}_u &= -\nabla_{q_u}V(q) \\
            \end{aligned}{}\right|_{\begin{array}{c}
                q_a = f(q_u,p_u) \\ 
                p_a = g(q_u,p_u) \\
            \end{array}}
            ,
    \end{equation}
    where
    \begin{multline}\label{eqn:g-qpu}
        g(q_u,p_u) \eqdef \left( \mathcal{A}(q,p_u)
        \begin{bmatrix} \Zmat{(n-k)\times k} \\ \Id{k} \end{bmatrix}\right)\inv 
        \cdot
        \Big(dh_{p_u}\nabla_{q_u}V(q) -
        \\
        \mathcal{A}(q,p_u)
        \begin{bmatrix} \Id{(n-k)} \\ \Zmat{k \times (n-k)} \end{bmatrix} p_u
        \left.\Big)\right|_{q_a = f(q_u,p_u)}
        .
    \end{multline}

\end{thm}
\begin{proof}
    Setting \(e = h(q,p_u)\) and using Assumption
    \ref{assm:inertially-actuated}, we find that
    \(\dot{e} = \mathcal{A}(q,p_u)p - dh_{p_u}\nabla_{q_u}V(q)\).
    Notice that
    \(\mathcal{A}(q,p_u)p = \mathcal{A}(q,p_u)[\Zmat{(n-k)\times k}; \Id{k}]p_a
    + \mathcal{A}(q,p_u)[\Id{n-k};\Zmat{k \times (n-k)}] p_u\).
    Since \(h(q,p_u)\) is regular,
    \(\mathcal{A}(q,p_u)[\Zmat{(n-k)\times k}; \Id{k}]\) is invertible.
    Taking \(e = \dot{e} = 0\), solving for \(p_a\), and setting 
    \(q_a = f(q_u,p_u)\) completes the proof.
\end{proof}



\textbf{Comparison with existing virtual constraint literature}: Horn \etal provide the constrained
dynamics for \vnhcs in \cite{nhvc_incline_walking}.
Their assumption \textbf{H2} is what we call regularity, and our requirement
that one can solve for \(q_a = f(q_u,p_u)\) on \(\Gamma\) implies their
assumption \textbf{H3} holds true.
Constraints of this form are used by Hamed and Ames in
\cite{nonholonomic_hybrid_zero_dynamics}.
In fact, one can show that the constrained dynamics \eqref{eqn:qpu-dynamics}
are exactly the constrained dynamics in
\cite[Eqn. (9)]{nonholonomic_hybrid_zero_dynamics}, which coincide with 
\cite[Eqn. (17)]{hybrid_zero_dynamics_bipedal_nhvcs}
when one chooses the special case \(\theta_1 = q_u\) and 
\(\theta_2 = p_u\).
The only real distinction between this section and the work of
\cite{nonholonomic_hybrid_zero_dynamics,hybrid_zero_dynamics_bipedal_nhvcs} is
that our constrained dynamics are derived in the Hamiltonian framework. 

\section{The Proposed Acrobot \vnhc}\label{sec:acrobot}

The goal in this article is to design a \vnhc which injects energy into the
acrobot by means of a giant-like motion.
Recall that the acrobot in Figure \ref{fig:acrobot} has dynamics given by
\eqref{eqn:acrobot-hamiltonian}, repeated here for convenience:
    \begin{equation*}
     \begin{cases}
        \dot{q} = \Minv(q) p 
        ,\\
        \dot{p}_u = -mgl\left(2s_u + s_{ua}\right) 
        ,\\
        \dot{p}_a =-\frac{1}{2}p\tpose \nabla_{q_a}\Minv(q) p
        - mgl s_{ua} + \tau.
    \end{cases}
\end{equation*}

Since the control input \(\tau\) only affects the actuated momentum,
the system above is already in simply actuated form.
Its state space is \(\Sone \times \Sone \times \Re \times \Re\).
We can therefore apply the theory from Section \ref{sec:vnhc} to design
a \vnhc of the form \(q_a = f(q_u,p_u)\) (i.e., a \vnhc 
\(h(q,p_u) = q_a - f(q_u,p_u) = 0\)).
Since we need the \vnhc to be regular, the following proposition will be
useful.
\begin{prop}\label{prop:acrobot-fpu-regular}
    Any relation \(q_a = f(p_u)\) 
    with \(f \in C^2\left(\R; \Sone\right)\) is a regular
    \vnhc of order 1 for the acrobot in \eqref{eqn:acrobot-hamiltonian}.
\end{prop}
\begin{proof}
    The decoupling matrix \eqref{eqn:decoupling-matrix} for the acrobot
    evaluates to
    \(((1+c_a)\partial_{q_u}f(q_u,p_u) + (3+2c_a))/(ml^2(2-c_a^2))\).
    Since \(\partial_{q_u} f = 0\), this function is positive for all values
    of \(q_a\), and hence is full rank \(1\) everywhere on the constraint manifold.
\end{proof}

To design the \vnhc, we begin by examining a person on a
seated swing.
The person extends their legs when the swing moves forward, and retracts their
legs when the swing moves backward.
As the swing gains speed, the person leans their body while
extending their legs higher, thereby shortening the distance
from their centre of mass to the pivot and adding more energy to the swing
\cite{how_to_pump_a_swing}.

Now imagine the person's torso is affixed to the swing's rope so they are
always upright. 
Imagine further that the swing has no seat, allowing the person to extend
their legs beneath them. 
This position is identical to that of a gymnast on a bar.

The acrobot's legs are rigid rods which cannot retract, so we emulate the person
on a swing by pivoting the legs toward the direction of motion. 
Since a person lifts their legs higher at faster speeds, the acrobot's legs should
pivot to an angle proportional to the swing's speed.
A real gymnast cannot swing their legs in full circles, though they
are usually flexible enough to raise them parallel to the floor;
hence, the \vnhc must restrict the leg angle \(q_a\) 
to lie in \([-Q_a, Q_a]\) for some \(Q_a \in [\frac{\pi}{2}, \pi[\). 
Because the direction of motion is entirely determined by \(p_u\), 
one \vnhc which emulates this process is \(q_a = \bar{q}_a\arctan( I p_u)\),
displayed in Figure \ref{fig:qa-arctan}.
Here, \(\bar{q}_a \in \, ]0,2 Q_a/\pi]\) and \(I \in \R\) are 
design parameters.

This constraint does not perfectly recreate giant motion, during which
the gymnast's legs are almost completely extended \cite{usagym_giant}.
It instead pivots the legs partially during rotations.
However, the behaviour is similar enough that this constraint will
still inject energy into the acrobot.
The final \vnhc is
\begin{equation}\label{eqn:acrobot-constraint}
    h(q,p) = q_a - \bar{q}_a \arctan(I p_u)
    .
\end{equation}

\begin{figure}
    \centering
    \includegraphics[width=0.7\linewidth]{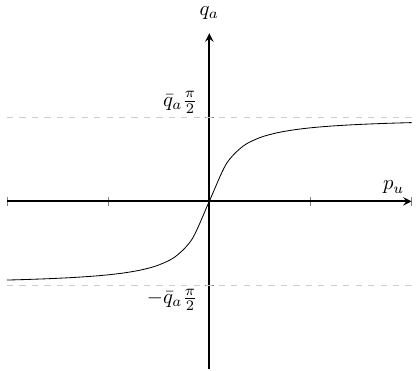}
    \caption{The acrobot constraint \(q_a = \bar{q}_a \arctan(I p_u)\).}
    \label{fig:qa-arctan}
\end{figure}

Recall that \((q_u, p_u)\) denote the angle and momentum of the acrobot's torso.
By Theorem \ref{thm:zero-dynamics}, the constrained dynamics arising from the
\vnhc \eqref{eqn:acrobot-constraint} are parameterized fully by 
\((q_u,p_u) \in \SxR\).
Here, \eqref{eqn:g-qpu} reduces to
\begin{multline*}
    g(q_u,p_u) = \frac{
    (1+c_a)p_u}{ml^2(3+2c_a)}
    - \frac{mgl\bar{q}_a I (2-c_a^2)(2s_u + s_{ua})
    }{(3+2c_a)(1+I^2 p_u^2)}
    ,
\end{multline*}
and the constrained dynamics \eqref{eqn:qpu-dynamics} are
\begin{equation}\label{eqn:acrobot-constrained-dynamics}
    \begin{cases}
    \dot{q}_u &= \frac{(1+I^2 p_u^2)p_u + m^2gl^3\bar{q}_a I(2s_u + s_{ua})(1+c_a) }
            {ml^2(1+I^2 p_u^2)(3+2c_a)}
    ,    \\
    \dot{p}_u &= - m g l (2s_u + s_{ua})
    ,
    \end{cases}
\end{equation}
subject to \(q_a = \bar{q}_a \arctan(I p_u)\). In what follows, we denote by $x=(q_u,p_u)$ the state of the constrained dynamics~\eqref{eqn:acrobot-constrained-dynamics}, and by $x(t,x_0)$ the solution of~\eqref{eqn:acrobot-constrained-dynamics} at time $t$ with initial condition $x(0)=x_0$. We let $x([a,b],x_0)$ denote the set $\{x(t,x_0): t \in [a,b]\}$.

\subsection{Energy Injection and Dissipation with the Proposed \vnhc}\label{sec:energy-inject-diss}
Suppose for a moment that \(I = 0\) in \eqref{eqn:acrobot-constraint}, i.e.,
that the legs stay fully extended.
The acrobot becomes a nominal pendulum with two masses, whose total mechanical
energy is
\begin{equation}\label{eqn:nominal-energy}
    E(q_u,p_u) \eqdef \frac{p_u^2}{10ml^2} + 3mgl(1 - \cos(q_u))
    .
\end{equation}
The function $E(q_u,p_u)$ is a first integral for the constrained
dynamics~\eqref{eqn:acrobot-constrained-dynamics}, i.e., each solution of the
constrained dynamics is confined to a level set of $E$, and each orbit not
containing equilibria is a connected component of a level set of $E$. Given
scalars $R$ and $R_1<R_2$, and letting $\# \in \{$
``$<$''$,$``$\leq$''$,$``$>$''$,$``$\geq$''$\}$, we define the following sets
(see Figure~\ref{fig:domains}):
\begin{equation}\label{eq:E_sets}
\begin{aligned}
E_R &\eqdef \{(q_u,p_u): E(q_u,p_u) = R\}, \\
E_{\# R} &\eqdef \{(q_u,p_u): E(q_u,p_u) \#R\}, \\
E_{[R_1,R_2]} &\eqdef \{(q_u,p_u): R_1 \leq E(q_u,p_u) \leq R_2\}.
\end{aligned}
\end{equation}
\begin{figure}[htb]
    \centering
    \begin{subfigure}[t]{0.45\linewidth}
        \includegraphics[width=\linewidth]{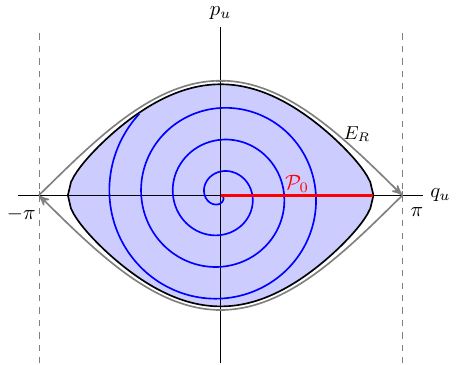}
        \caption{The set $E_{\leq R}$ with $0 < R < \bar R$ and an oscillation gaining energy.} 
        \label{fig:acrobot-oscillation-domain}
    \end{subfigure}
    \hfill
    \begin{subfigure}[t]{0.45\linewidth}
        \includegraphics[width=\linewidth]{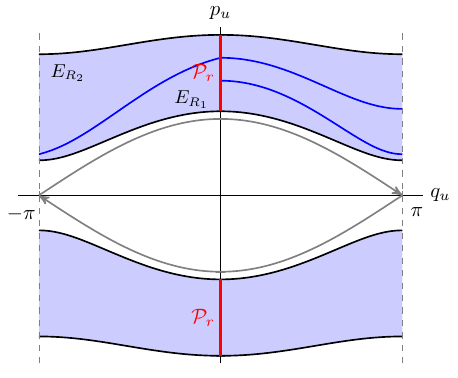}
        \caption{The set $E_{[R_1,R_2]}$, with $\bar R < R_1 < R_2$, and a rotation gaining energy.}
            \label{fig:acrobot-o2}
    \end{subfigure}
    \caption{Oscillations and rotations gaining energy. 
}
\label{fig:domains}
\end{figure}

Letting $\bar R \eqdef E(\pi,0)$, solutions initialized in the interior of the
set $E_{\leq \bar R}$ correspond to {\em oscillations} of the pendulum, i.e.,
rocking motions where the pendulum does not perform full revolutions. On the
other hand, for each $R>\bar R$, solutions initialized in the interior of the
set $E_{[\bar R,R]}$ correspond to {\em rotations} of the pendulum, i.e., full
revolutions around the pivot point. 

Now consider the case $I\neq 0$, i.e., when the acrobot is not fully extended.
The constrained dynamics~\eqref{eqn:acrobot-constrained-dynamics} are no longer
Hamiltonian, and their solutions are no longer confined to level sets of $E$. In
what follows, we will show that the orbits of the constrained dynamics roughly
speaking gain energy or lose energy depending on the sign of the design
parameter $I$. To be more precise in describing this energy injection or
dissipation property, we next define what is an oscillation or a rotation
gaining or losing energy. Intuitively, oscillations gaining energy correspond to
rocking motions of the torso with increasing amplitude, while rotations gaining
energy correspond to full revolutions with increasing speed.
Figure~\ref{fig:domains} illustrates these concepts.

We define two Poincar\'e sections,
\[
\begin{aligned}
\cP_o & \eqdef \{(q_u,p_u) \in E_{<\bar R}: q_u >0, p_u=0\} \\
\cP_r & \eqdef \{(q_u,p_u) \in E_{> \bar R}: q_u=0\}.
\end{aligned}
\]
Consider a solution $x(t,x_0)$ of the constrained dynamics~\eqref{eqn:acrobot-constrained-dynamics}. We say that $x(t,x_0)$ is an {\em oscillation over $[0,\bar t]$} if $x([0,\bar t],x_0) \cap \cP_o$ is a discrete set, i.e., a set of isolated points. Similarly, $x(t,x_0)$ is a {\em rotation over $[0,\bar t]$} if $x([0,\bar t],x_0) \cap \cP_r$ is a discrete set. The isolated points mentioned above represent consecutive crossings of the Poincar\'e sections $\cP_o$ and $\cP_r$ by the solution $x(t,x_0)$ at times $t^i_o$, $t^i_r$, where $t^i_{o,r} < t^{i+1}_{o,r}$. 

Letting $x^i_o(x_0) \eqdef x(t^i_o,x_0)$ and $x^i_r(x_0) \eqdef x(t^i_r,x_0)$, we say that {\em an oscillation over $[0,\bar t]$ gains energy (respectively, loses energy)} if the sequence $\{\|x^i_o(x_0)\|\}$ is monotonically increasing (respectively, monotonically decreasing) with $i$. 
Similarly, {\em a rotation over $[0,\bar t]$ gains energy (respectively, loses energy)} if the sequence $\{\|x^i_r(x_0)\|\}$ is monotonically increasing (respectively, monotonically decreasing) with $i$.

The next theorem states that if the design parameter $I$ in the \vnhc~\eqref{eqn:acrobot-constraint} is chosen small enough, then the constrained dynamics have oscillations gaining or losing energy, depending on the sign of $I$.

\begin{thm}[oscillations gaining/losing energy]\label{thm:acrobot-oscillations}
 Consider the constrained dynamics in~\eqref{eqn:acrobot-constrained-dynamics}. For each $R_1<R_2 \in \mathopen] 0, \bar R \mathclose[$, there exists \(I^\star > 0\) such that for each \(I \in \, ]0,I^\star]\) (respectively, $I \in [-I^\star,0\mathclose[$) and $x_0 \in E_{[R_1,R_2]}$, there exists $\bar t\geq 0$ such that $x(\bar t,x_0) \in E_{R_2}$ (respectively, $x(\bar t, x_0) \in E_{R_1}$), and $x(t,x_0)$ is an oscillation gaining energy (respectively, losing energy) over $[0, \bar t]$. 
\end{thm}

\begin{proof}
 See Section~\ref{sec:proof}.
\end{proof}

The next result concerns rotations.  In preparation for the theorem statement, we define 
    \[
        b(r,\theta) \eqdef 
        \frac{5C \left(
        \frac{C}{\bar{q}_a}\left(18s_\theta^2 + 30c_\theta(1 - c_\theta)\right)
            - c_\theta r^2
        \right)}{
        | r|\sqrt{ r^2 - 30m^2gl^3(1 - c_\theta)}
        },
\]
with \(C = m^2gl^3\bar{q}_a\), and 
        \(S( r) \eqdef \int \limits_{0}^{2\pi} b(r,\theta)d\theta\).

\begin{thm}[rotations gaining/losing energy]\label{thm:acrobot-rotations}
    Consider the constrained dynamics
    in~\eqref{eqn:acrobot-constrained-dynamics}, and fix $R_1, R_2$ such that
    \(\bar R<R_1<R_2\). Suppose there exists \(\varepsilon > 0\) such that \(S(
    r) \geq \varepsilon\) for all \( r \in \, [(10 m l^2 R_1)^{1/2}, (10 m l^2
    R_2)^{1/2}]\). Then there exists \(I^{\star}>0\) such that, for each \(I \in
    \, ]0,I^{\star}]\) (respectively, $I \in [-I^\star,0\mathclose[$) and $x_0
    \in E_{[R_1,R_2]}$, there exists $\bar t \geq 0$ such that $x(\bar t,x_0)
    \in E_{R_2}$  (respectively, $x(\bar t, x_0) \in E_{R_1}$), and $x(t,x_0)$
    is a rotation gaining energy (respectively, losing energy) over $[0, \bar
    t]$. 
\end{thm}
\begin{proof}
See Section~\ref{sec:proof}.
\end{proof}

\begin{rem}
Theorem~\ref{thm:acrobot-oscillations} concerning oscillations
gai\-ning/lo\-sing energy states that for any set of physical parameters, the
acrobot constrained by
\eqref{eqn:acrobot-constraint} will gain enough energy to exit (in finite time) a
level set of the function $E$ arbitrarily close to the boundary $E_{\bar R}$
separating oscillations from rotations, provided the parameter $I$
in~\eqref{eqn:acrobot-constraint} is chosen small enough. The result is
semiglobal relative to the set $E_{<\bar R}$ in that energy gain can be achieved
on any compact subset $E_{[R_1,R_2]}$ of $E_{<\bar R}$ by a suitably small $I$.
Vice versa, for small {\em negative} $I$, the oscillations will lose energy on any compact subset $E_{[R_1,R_2]}$ of $E_{<\bar R}$.

Theorem~\ref{thm:acrobot-rotations} concerning rotations relies on an assumption on the function $S(r)$ which depends on the physical parameters of the acrobot, as well as the design parameter $\bar q_a$. The result is otherwise analogous to Theorem~\ref{thm:acrobot-oscillations}.
\end{rem}

\begin{rem}\label{rem:thm:caveat}
The theorems above state that under certain conditions, if the design parameter $I$ is chosen small enough then the acrobot will exhibit oscillations and rotations gaining energy. The oscillations will exit a compact subset of $E_{<\bar R}$ in finite time, while the rotations will exit a compact subset of $E_{>\bar R}$. Will the oscillations gaining energy eventually turn into rotations gaining energy? The theorem does not answer this question, but extensive simulations and physical experiments presented below suggest that the answer is yes. 
\end{rem}
%
%
%
%


\subsection{Energy Regulation}\label{sec:energy-reg}
One can apply the results of Theorems~\ref{thm:acrobot-oscillations} and~\ref{thm:acrobot-rotations}
towards energy regulation; 
that is, one can stabilize oscillations or rotations by appropriately toggling
between injection and dissipation \vnhcs, which can be achieved by changing the
sign of \(I\) in \eqref{eqn:acrobot-constraint}.

\textbf{Rotation Regulation}: choose a desired rotation rate 
\(p_\text{des} > 0\) and a hysteresis value \(\delta \in [0,1]\) such that \((1-\delta) p_\text{des} > \sqrt{60m^2gl^3}\). Each time the orbit crosses the \(p_u\)-axis (i.e. when \(q_u = 0\)), the
supervisor changes which \vnhc is enforced as follows:
\begin{itemize}
    \item If \(|p_u| < (1-\delta)p_\text{des}\), enable the injection \vnhc by
        setting \(I > 0\) in \eqref{eqn:acrobot-constraint}.
    \item If \(|p_u| > (1+\delta)p_\text{des}\), enable the dissipation \vnhc by
        replacing \(I > 0\) with \(-I < 0\) in \eqref{eqn:acrobot-constraint}.
    \item If \((1-\delta)p_\text{des} \leq |p_u| \leq (1+\delta)p_\text{des}\),
        extend the leg fully by setting \(q_a = 0\).
        In simulation we assume this can be done instantaneously,
        though in practice this takes time.
\end{itemize}

All orbits of the acrobot initialized in $E_{< \bar R}$ must cross the \(p_u\) axis. The procedure above exploits this observation to stabilize a rotation even if the acrobot is initialized in $E_{< \bar R}$, although in this case we do not offer a theoretical guarantee (see Remark~\ref{rem:thm:caveat}). While the above supervisor is designed to regulate a rotation rate, it
does not impose the rotation direction.


\textbf{Oscillation Regulation}: Choose a desired oscillation angle  \(q_\text{des} \in \, ]0,\pi[\) and, to avoid infinite switching, a hysteresis value \(\delta \in [0,\pi/q_\text{des} - 1]\). An orbit in the \((q_u,p_u)\)-plane will either cross the \(q_u\) axis if the orbit corresponds to a rocking motion, or it will cross the line \(|q_u| = \pi\) if the orbit corresponds to a full revolution. When either of these occur, the supervisor does the following: 
\begin{itemize}
    \item If \(|q_u| < (1-\delta)q_\text{des}\), enable the injection \vnhc.
    \item If \(|q_u| > (1+\delta)q_\text{des}\), enable the dissipation \vnhc.
    \item If \((1-\delta)q_\text{des} \leq |q_u| \leq (1+\delta)q_\text{des}\),
        keep the leg extended at \(q_a = 0\). This can be done continuously since
        \(q_a = 0\) when \(p_u = 0\).
\end{itemize}
Note that by the choice of \(\delta\), if the supervisor kicks in when  \(|q_u| = \pi\) (i.e., when the robot is rotating) then the supervisor will automatically enable the dissipation \vnhc.

\section{Simulation Results}\label{sec:simulations}
In Section \ref{sec:experiments} we will test the \vnhc~\eqref{eqn:acrobot-constraint} on a physical acrobot. While the acrobot model used for the theoretical development assumes point-masses and links of equal lengths, the physical acrobot has distributed mass and unequal lengths. In preparation for the experiments, in this section we simulate the proposed \vnhc controller on a more accurate acrobot model, depicted in Figure~\ref{fig:acrobot-model}, with distributed mass and unequal link lengths. Let
\begin{align*}
    m_{11}(q) \eqdef &m_a l_u^2 + 2m_al_ul_{c_a}\cos(q_a) + m_al_{c_a}^2 +
    m_ul_{c_u}^2 \\
                 &+ J_u + J_a
                 , \\
    m_{12}(q) \eqdef &m_al_{c_a}^2 + m_al_ul_{c_a}\cos(q_a) + J_a
    , \\
    m_{22}(q) \eqdef &m_al_{c_a}^2 + J_a,
\end{align*}
where \(J_u\) and \(J_a\) are the moments of inertia of the torso and leg links
respectively, and all other parameters are illustrated in Figure~\ref{fig:acrobot-model}. The distributed mass acrobot has inertia matrix
\begin{equation*}
    M(q) = \begin{bmatrix}
        m_{11}(q) & m_{12}(q) \\
        m_{12}(q) & m_{22}(q)
    \end{bmatrix}
    ,
\end{equation*}
and potential function
\begin{equation*}
    V(q) = g\big(m_al_{c_a}(1-c_{ua})
    + (m_al_u + m_ul_{c_u})(1-c_u)\big)
    .
\end{equation*}

\begin{figure}
    \centering
    \includegraphics[width=0.8\linewidth]{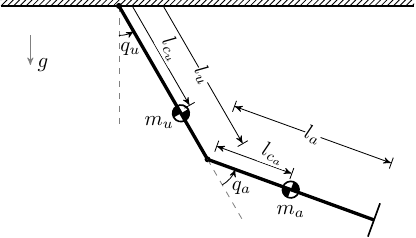}
    \caption{The distributed mass acrobot model, represented by two weighted rods
    differing in both length and mass.}%
    \label{fig:acrobot-model}
\end{figure}

Theorems~\ref{thm:acrobot-oscillations} and~\ref{thm:acrobot-rotations} make use of the mechanical energy \eqref{eqn:nominal-energy} of the nominal pendulum obtained by setting \(I = 0\) in \eqref{eqn:acrobot-constraint}. Using the parameters in Table~\ref{tab:acrobot-parameters} and the inertia matrix and potential function given above, we get the mechanical energy of the nominal pendulum with $I=0$:
\[
    E(q_u,p_u) \approx 396.5501 p_u^2 + 0.5997(1 - \cos(q_u))
    .
\]
This is the function we use in interpreting the simulation results presented next.

We will now simulate the effects of constraining the physical acrobot with the
\vnhc \eqref{eqn:acrobot-constraint}, thereby demonstrating that \vnhcs are
robust to model mismatch. According to Theorems~\ref{thm:acrobot-oscillations}
and~\ref{thm:acrobot-rotations}, the control parameter \(I\) must be ``small"
for the \vnhc~\eqref{eqn:acrobot-constraint} to inject (or dissipate) energy
into the simplified acrobot. The theorems do not specify how small \(|I|\) must
be; while we could make it arbitrarily small in simulations, we will eventually
implement this \vnhc on a physical testbed where \(|I|\) must be large enough to
overcome friction. 

Setting \(\bar{q}_a = 1\), we experimentally determined that \(|I| = 10\) is a viable control parameter, so this is the value we will use for all simulations and experiments. In other words, the \textit{injection \vnhc} is \eqref{eqn:acrobot-constraint} with  \(I = 10\) while the \textit{dissipation \vnhc} is \eqref{eqn:acrobot-constraint} with  \(I = -10\).

\begin{table}
    \centering
    \caption{Physical parameters for the real acrobot.}
    \label{tab:acrobot-parameters}
    \resizebox{\columnwidth}{!}{%
    \begin{tabular}{ccccccccc}
        \toprule
        $m_u$ & $m_a$ & $l_u$ & $l_a$ & $l_{c_u}$ & $l_{c_a}$ & $J_u$ & $J_a$ & $g$ \\
        (kg) & (kg) & (m) & (m) & (m) & (m) & (kg\(\cdot\)m\(^2\)) &
        (kg\(\cdot\)m\(^2\)) & (m/s\(^2\)) \\
        \midrule
        0.2112 & 0.1979 & 
        0.148 & 0.145 & 
        0.073 & 0.083 & 
        0.00129 & 0.00075 & 
        9.81 \\ 
        \bottomrule
    \end{tabular}
    } 
\end{table}

\subsection{Energy Injection}

In simulation, we stabilized the injection \vnhc for the acrobot using the
controller \eqref{eqn:stabilizing-controller}.
We initialized the acrobot on the constraint manifold
with initial condition \((q_u,p_u) = \left(\pi/32,0 \right)\) and simulated the
constrained system for \(30\) seconds.
The resulting orbit is plotted in Figure
\ref{fig:acrobot-in-orbit}.

The level set \(E_{\bar R}\), with \(\bar R=E(\pi,0)\), is outlined in black.
Recall that this level set is the boundary between oscillations and rotations of the
nominal pendulum.
The points where the orbit exits \(E_{\bar R}\) are marked with black asterisks,
with the final departure marked by a red asterisk.
Interestingly, the choice of \(I\) is large enough that we observe significant
differences between the nominal pendulum and the constrained dynamics:
\(E_{\bar R}\) intersects the \(p_u\)-axis at \(\abs{p_u} \approx 0.17\), yet the
constrained acrobot begins rotating once it hits the
\(p_u\)-axis at \(\abs{p_u} \approx 0.16\). 
This indicates that higher values of \(I\) enable the acrobot to gain energy
faster and begin rotating sooner, so long as the actuator does not saturate.

\begin{figure}[]
    \centering
    \includegraphics[width=0.8\linewidth]{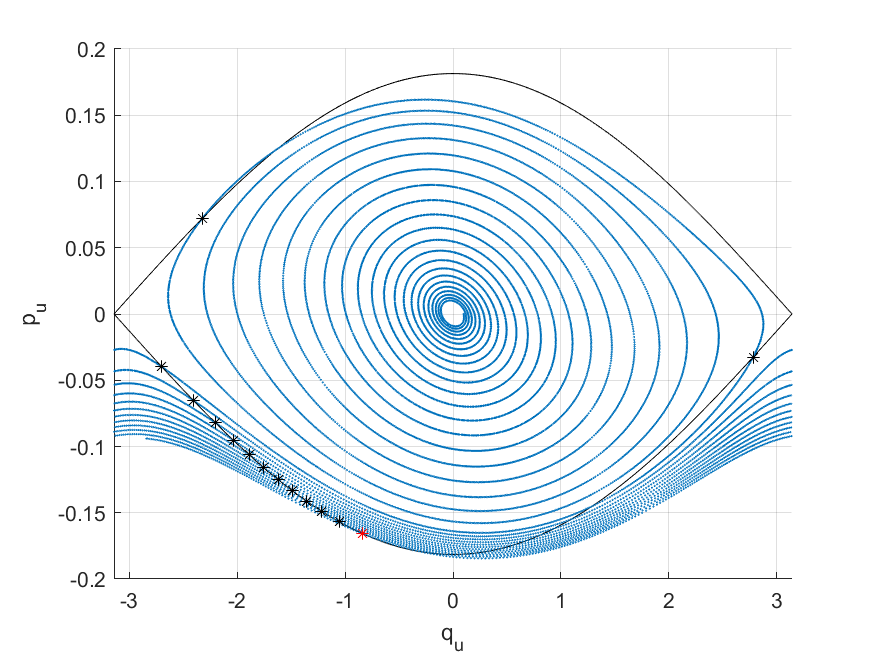}
    \caption{A simulation of the acrobot gaining energy.}
    \label{fig:acrobot-in-orbit}
\end{figure}

To verify numerically that the acrobot would consistently achieve rotations, we
ran a Monte-Carlo \cite{montecarlo} simulation where we initialized the acrobot
randomly inside the sublevel set
\[
    \left\{(q_u,p_u) \in \SxR \mid
    E(q_u,p_u) \leq E\left(\frac{\pi}{32},0\right)\right\}
    ,
\] 
and measured how long it took to begin rotating.
The results in Figure \ref{fig:acrobot-mc} show that
the acrobot always rotated within 20--40 seconds.

\begin{figure}[]
    \centering
    \includegraphics[width=0.8\linewidth]{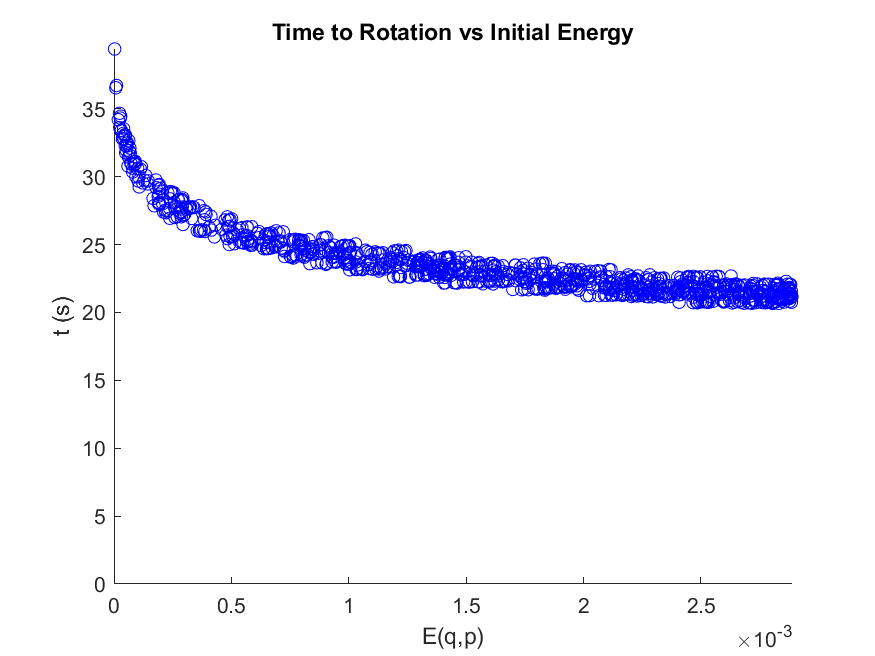}
    \caption{Monte Carlo simulation for energy injection.}
    \label{fig:acrobot-mc}
\end{figure}

\subsection{Energy Dissipation}

In simulation, we stabilized the dissipation \vnhc and initialized the acrobot on
the constraint manifold with a rotation \((q_u,p_u) = (0,0.18)\).
We simulated the constrained system for \(30\) seconds and plotted the
resulting orbit in Figure \ref{fig:acrobot-diss-orbit}. 
As expected, the acrobot slows down over time.
We highlight the locations where the orbit crossed the set \(E_{\bar R}\) by
black asterisks, with the final crossing in red.
After this final crossing, the acrobot ceased rotating and its oscillations
decayed to zero.

\begin{figure}
    \centering
    \includegraphics[width=0.8\linewidth]{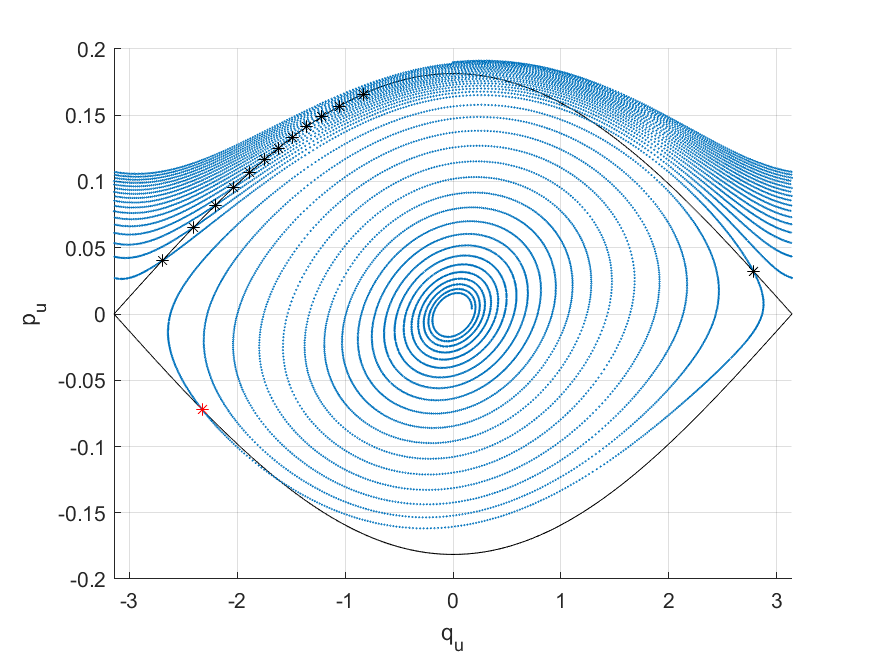}
    \caption{A simulation of the acrobot dissipating energy.}
    \label{fig:acrobot-diss-orbit}
\end{figure}

\subsection{Oscillation Regulation}

Recall from Section \ref{sec:energy-reg} that one can use a supervisor to stabilize
oscillations by appropriately toggling between injection and dissipation \vnhcs
whenever the orbit of the acrobot crosses the \(q_u\)-axis.
Figure \ref{fig:acrobot-osc-reg} shows the supervisor stabilizing an
oscillation with body angle \(q_\text{des} = \pi/2\) and a 
\(5\%\) hysteresis, meaning \(\delta = 0.05\).
The supervisor reevaluated its choice of \vnhc at each black asterisk;  
the red contour corresponds to the part of the orbit where the supervisor kept
the leg extended, because the oscillation was within tolerance of
\(q_\text{des}\).
The solid black line is the desired oscillation, and the dashed black
lines show the hysteresis around that orbit.

In Figure \ref{fig:acrobot-osc-reg-in} the acrobot was initialized at 
\((q_u,p_u) = (\pi/32,0)\); here the supervisor injected energy to stabilize the
desired orbit.
In Figure \ref{fig:acrobot-osc-reg-diss} the acrobot was initialized at the rotation
\((q_u,p_u) = (0,0.19)\); note that the supervisor is first enabled at the line
\(|q_u| = \pi\) and dissipates energy. Eventually the acrobot begins a rocking
motion and the supervisor is enabled again each time the acrobot hits the
\(q_u\)-axis. 
The supervisor continues to dissipate energy until it stabilizes the desired
oscillation.

\begin{figure}
    \centering
    \begin{subfigure}[t]{0.49\linewidth}
        \includegraphics[width=\linewidth]{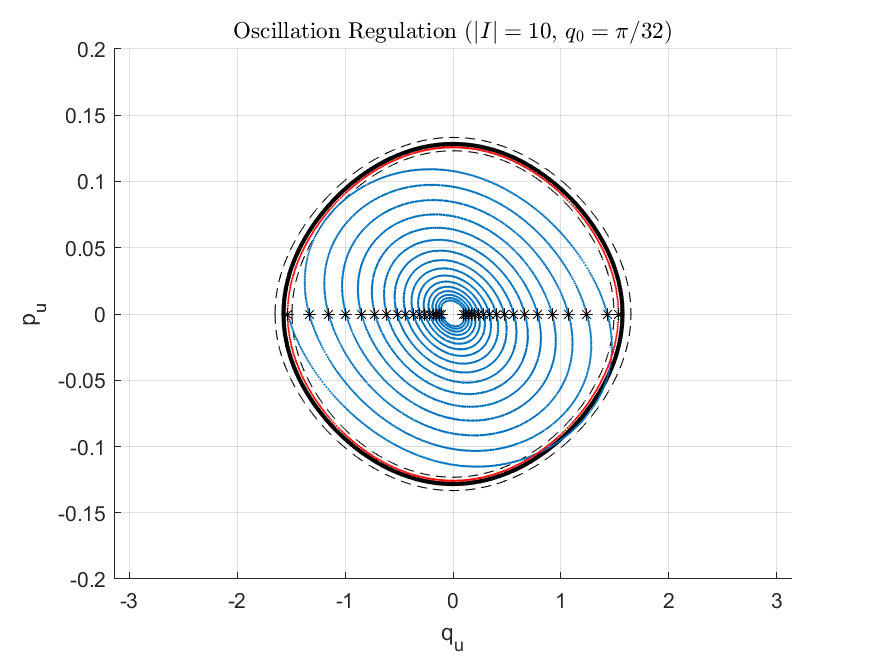}
        \caption{Stabilizing an oscillation from below.}
        \label{fig:acrobot-osc-reg-in}
    \end{subfigure}
    \begin{subfigure}[t]{0.49\linewidth}
        \includegraphics[width=\linewidth]{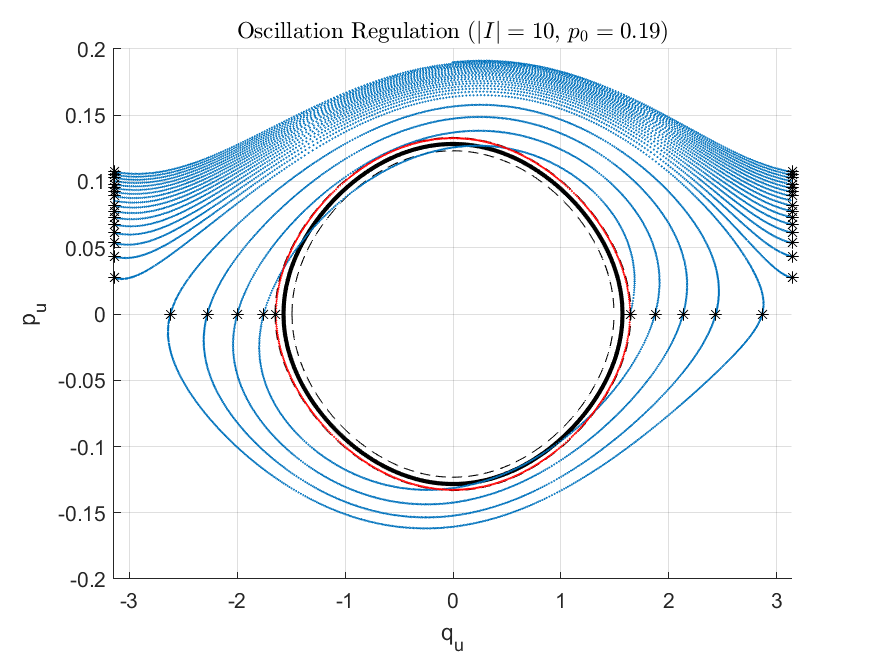}
        \caption{Stabilizing an oscillation from above.}
        \label{fig:acrobot-osc-reg-diss}
    \end{subfigure}
    \caption{Using a supervisor to stabilize the oscillation with peak angle
    \(q_\text{des} = \pi/2\). The desired oscillation is depicted with a solid
    black line.}
    \label{fig:acrobot-osc-reg}
\end{figure}

\subsection{Rotation Regulation}
One can also use a supervisor to stabilize rotations through the mechanism
described in Section \ref{sec:energy-reg}, where the supervisor toggles between
injection and dissipation \vnhcs at each crossing of the \(p_u\)-axis.
Rotation regulation for the acrobot is demonstrated in 
Figure \ref{fig:acrobot-rot-reg}, where the supervisor stabilizes 
\(p_\text{des} = 0.19\) with a \(2\%\) hysteresis \(\delta = 0.02\).
The supervisor evaluated its choice of \vnhc at each black asterisk.
Once it was within range of \(p_\text{des}\) it extended the legs completely,
the orbit of which is shown in red.

In Figure \ref{fig:acrobot-rot-reg-in} the acrobot was initialized at 
the oscillation \((q_u,p_u) = (\pi/2,0)\); here the supervisor injected energy
until the orbit hit the \(p_u\)-axis near \(p_\text{des}\).
In Figure \ref{fig:acrobot-rot-reg-diss} the acrobot was initialized at the
(fast) rotation \((q_u,p_u) = (0, 0.23)\); here the supervisor dissipated energy.
In both cases, the desired rotation was stabilized correctly.

\begin{figure}
    \centering
    \begin{subfigure}[t]{0.49\linewidth}
        \includegraphics[width=\linewidth]{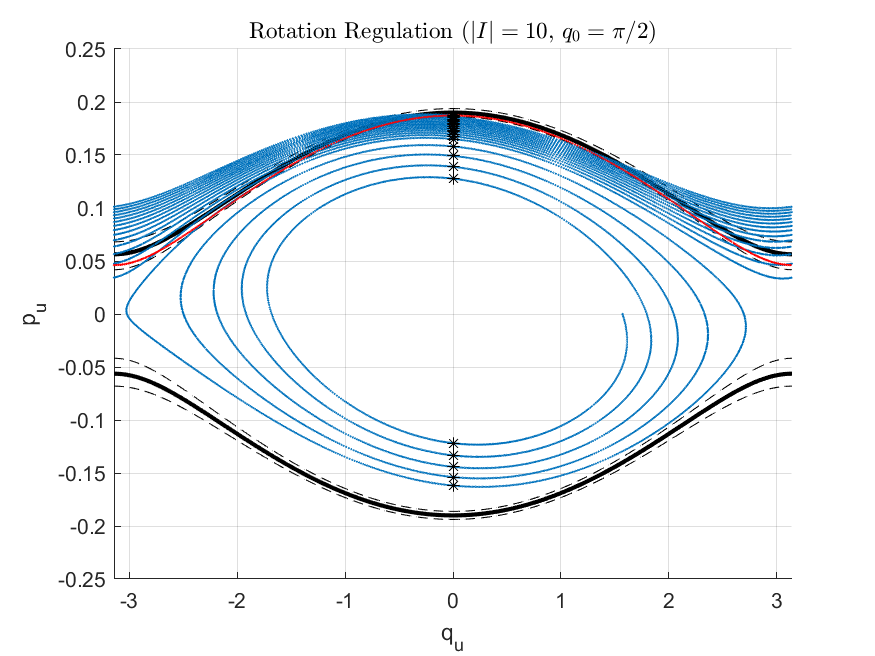}
        \caption{Stabilizing a rotation from below.}
        \label{fig:acrobot-rot-reg-in}
    \end{subfigure}
    \begin{subfigure}[t]{0.49\linewidth}
        \includegraphics[width=\linewidth]{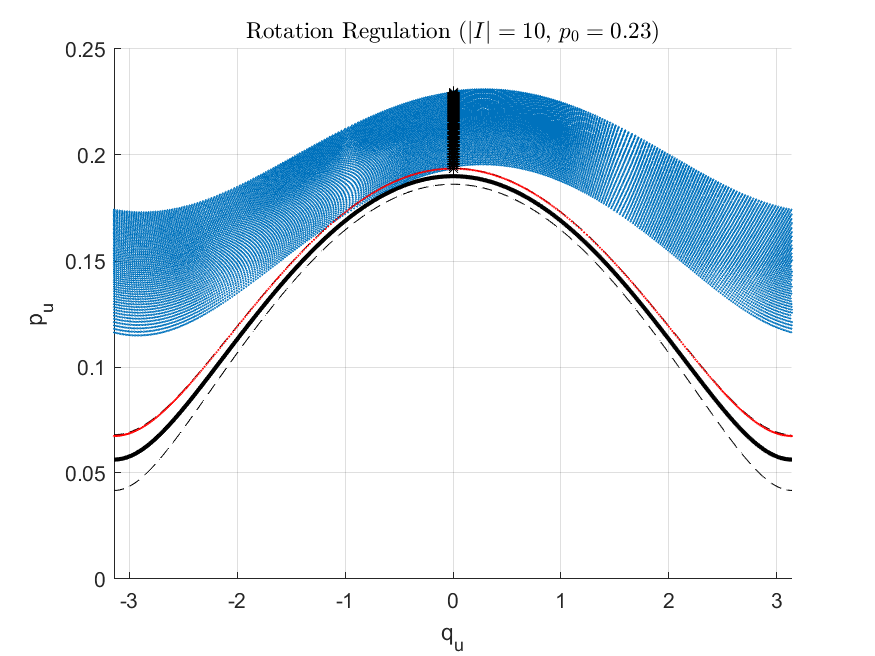}
        \caption{Stabilizing a rotation from above.}
        \label{fig:acrobot-rot-reg-diss}
    \end{subfigure}
    \caption{Using a supervisor to stabilize the rotation with maximal momentum
    \(p_\text{des} = 0.23\). The desired rotation is depicted with a solid black
    line.}
    \label{fig:acrobot-rot-reg}
\end{figure}

Note the difference in shape between the blue rotations of the dissipation \vnhc
and the red rotation of the nominal pendulum in Figure
\ref{fig:acrobot-rot-reg}: the red one slows down much more
near \(|q_u| = \pi\).
This difference arises because of the size of \(I\):
if \(|I|\) were smaller, the blue rotations would be more similar in shape to the
red one because the constrained dynamics for the dissipation \vnhc would be
well approximated by the nominal pendulum.

\subsection{Summary of Results}
The simulation results in this section demonstrate the energy regulation
capabilities of the proposed \vnhc.  
We were able to stabilize both oscillations and rotations by implementing a
control supervisor which toggled between injection, dissipation, and
leg-extension \vnhcs.
In particular, these simulations indicate that the proposed \vnhc
works even for acrobots whose limbs have differing masses and lengths.

\section{Physical Experiments}\label{sec:experiments}

\subsection{Hardware Description}
In this section we will demonstrate that the proposed \vnhc is robust to friction, sensor
noise, and other real-world considerations by testing it on the physical
acrobot depicted in Figure \ref{fig:xingbo-acrobot}.
This platform is called SUGAR, which stands for
Simple Underactuated Gymnastics and Acrobatics Robot.
Its dynamic parameters are outlined in Table \ref{tab:acrobot-parameters}.

\begin{figure}
    \centering
    \includegraphics[width=0.9\linewidth]{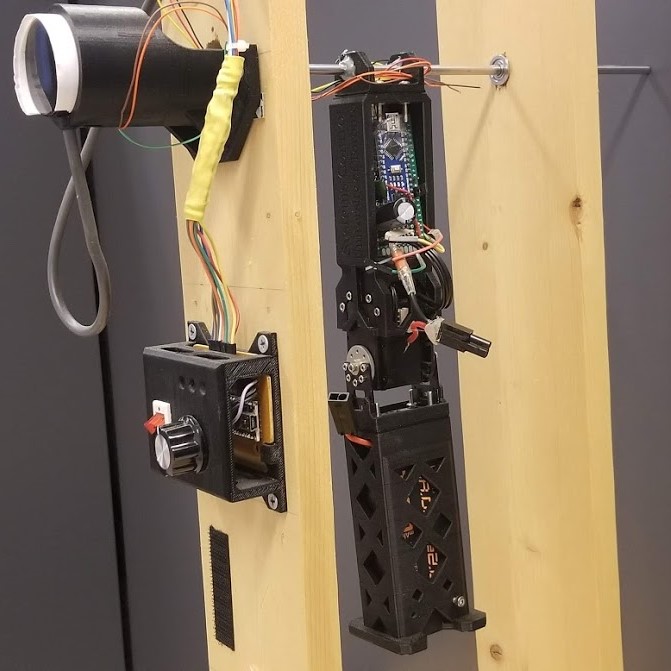}
    \caption{SUGAR is the physical acrobot built by Wang \cite{xingbo_thesis}.}
    \label{fig:xingbo-acrobot}
\end{figure}

SUGAR is comprised of two 3D-printed links: a torso and a leg.
The torso houses an Arduino Nano microcontroller unit (MCU) which controls
a Dynamixel RX24F servo motor between the torso and the leg.
The MCU and the motor are powered by a 12V battery held in a compartment
in the leg.

The torso is rigidly attached to a metal bar, which is held up by two wooden
posts.
On the exterior of one post is a control box with a power switch and a second
Arduino Nano 328.
The purpose of this control box is to read measurements from a rotary
encoder connected to the metal bar, and to transmit these measurements to the
MCU.
The two Arduinos communicate through wires attached to a slip ring on the metal
bar, and the signals are transmitted via I2C.
The control box also provides a USB interface which allows the user to read the
data from the acrobot in real time.

The rotary encoder directly measures \(q_u\), and the Dynamixel servo
motor provides measurements of \(q_a\).
However, there are no sensors measuring the velocity \(\dot{q}\), which means we cannot
directly evaluate \(p_u\) and \(p_a\).
To resolve this issue, the MCU estimates \(\dot{q}\) by applying a washout
filter to sequential measurements of \(q\).
We then compute \(p = M(q)\dot{q}\) for use in the \vnhc controller.

The communication speed between Arduinos restricts the sampling rate of \(p\) to
500Hz.
This low sampling rate results in a noisy momentum signal which suffers
from noticeable phase lag.
This also rate-limits the control signal to \(500\)Hz, which impacts any control
implementation.

Finally, the Dynamixel servo motor does not have a torque control mode; instead,
we can assign the servo setpoint at iteration \(k \in \mathbb{Z}_{> 0}\)
via \(q_a^{k} = \arctan(I p_u^{k-1})\).
This negatively affects the stabilization to the constraint
manifold because we are introducing timing errors from the servo's built-in PID
controller.

\subsection{Experimental Results}

We performed the following tests on SUGAR with the energy injection \vnhc.

\begin{enumerate}
    \item \textbf{Baseline Test:} 
    we initialized SUGAR at 
    \((q_u,p_u) \approx \left(\pi/8,0\right)\). 
    The resulting orbit in Figure \ref{fig:acrobot-unperturbed-orbit} shows
    that SUGAR clearly gaining energy over time.
    Its motion looks similar to that of the energy injection simulation (Figure
    \ref{fig:acrobot-in-orbit}), though its energy gain ceases once it reaches a
    rotation with energy \(E(0,0.195)\).
    The energy gain likely ceases because of friction at the pivot, which was
    not modelled in simulation.

\begin{figure}
    \centering
    \includegraphics[width=0.8\linewidth]{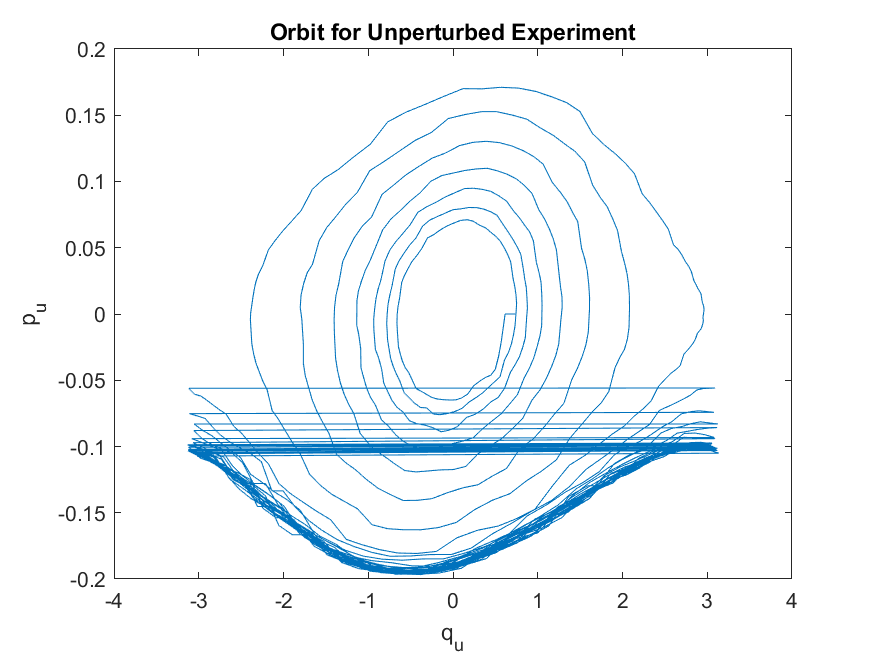}
    \caption{Baseline Test: SUGAR's baseline energy injection orbit. A video of this experiment is available at 
    {\footnotesize
\url{https://play.library.utoronto.ca/watch/6703e828acdcd38058ca884d13660a88}}.}
    \label{fig:acrobot-unperturbed-orbit}
\end{figure}

\item \textbf{Perturbation Test 1:}
    we initialized SUGAR at 
    \((q_u,p_u) \approx \left(\pi,0\right)\), let it run for 15
    seconds, then introduced a rod as SUGAR passed through the bottom of
    its arc.
    This caused a collision which stopped SUGAR in its tracks, at which
    point we immediately removed the rod so SUGAR could continue
    unperturbed.
    The resulting orbit is shown in Figure \ref{fig:acrobot-stopped-orbit}.
    The blue rotation curve corresponds to the orbit before the disturbance,
    while the red spiral confirms that SUGAR begins oscillating after it
    was stopped.  
    After the collision, SUGAR gains energy and eventually starts
    rotating again.

\begin{figure}
    \centering
    \includegraphics[width=0.8\linewidth]{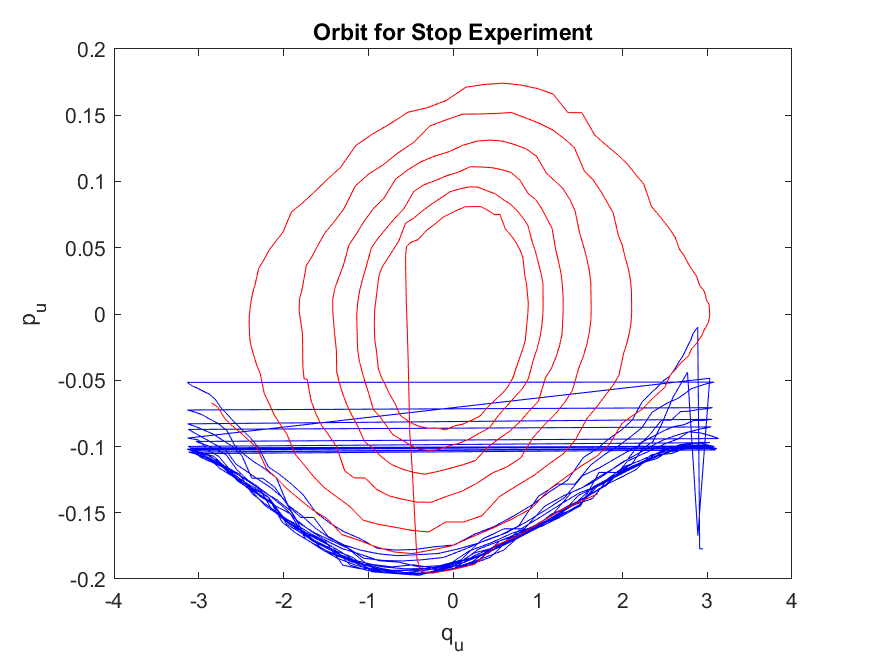}
    \caption{Perturbation Test 1: SUGAR's orbit before (blue) and after (red)
    stopping. A video of this experiment is available at {\footnotesize
\url{https://play.library.utoronto.ca/watch/164a21f0a52d4d56ef60bf1911545e0e}}.}
    \label{fig:acrobot-stopped-orbit}
\end{figure}

\item \textbf{Perturbation Test 2:}
    to see how SUGAR would respond when pushed, we allowed
    it to rotate unperturbed for 15 seconds and then pushed it in its
    direction of motion.
    The orbit in Figure \ref{fig:acrobot-fpush-orbit} shows that SUGAR,
    when pushed, rotates with energy \(E(0,-0.22)\), but then slows down until
    it reaches a rotation with energy \(E(0,-0.195)\).
    We repeated this test by pushing SUGAR against its direction of motion.
    The orbit in Figure \ref{fig:acrobot-rpush-orbit} demonstrates that SUGAR
    readily changes direction, and quickly achieves its maximum speed
    with energy \(E(0,0.195)\).
\end{enumerate}

In simulation, the acrobot was able to gain energy even when initialized
with energy \(E(q_u,p_u) > E(0,0.195)\). 
The baseline and push tests suggest that the \vnhc injects energy into SUGAR
only on \(\{E(q_u,p_u) \leq E(0,0.195)\}\).
This difference between simulation and implementation is likely due to
friction, as well as timing errors incurred by the PID controller in the servo
motor.

\begin{figure}
    \centering
    \begin{subfigure}[ht]{0.49\linewidth}
        \includegraphics[width=\linewidth]{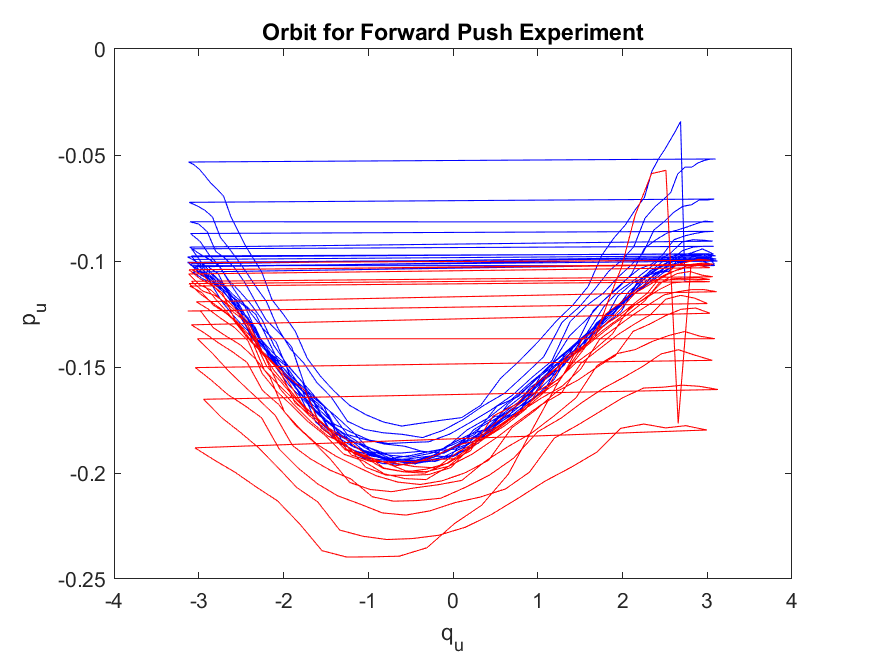}
        \caption{The forward push test.}
        \label{fig:acrobot-fpush-orbit}
    \end{subfigure}
    \begin{subfigure}[ht]{0.49\linewidth}
        \includegraphics[width=\linewidth]{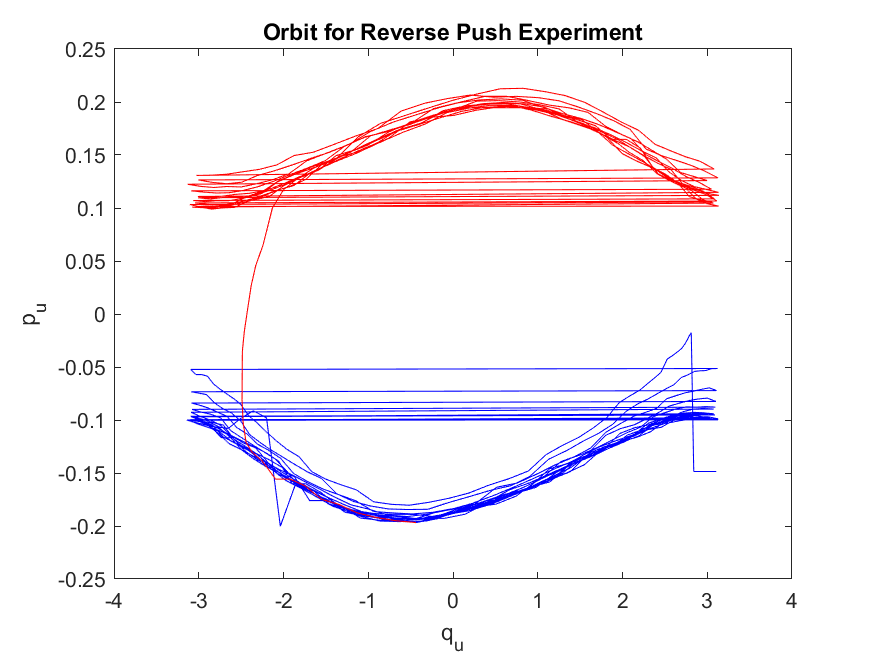}
        \caption{The reverse push test.}
        \label{fig:acrobot-rpush-orbit}
    \end{subfigure}
    \caption{Perturbation Test 2: SUGAR's orbit before (blue) and after (red)
    pushing. A video of this experiment is available at 
    {\footnotesize\url{https://play.library.utoronto.ca/watch/b13d4dc9f6daefa7f4d56dbf1b1bd482}.}
    }
\end{figure}

\subsection{Summary of Results} 
We performed three tests on SUGAR: a baseline energy injection test,
a stop test, and a push test.
These experiments demonstrate that \vnhc-based energy injection is robust to
significant model mismatch, friction, sensor noise, discretized control
implementation, rate-limited measurement and control signals, and dramatic
external disturbances.

\section{Proofs of Theorems~\ref{thm:acrobot-oscillations} and~\ref{thm:acrobot-rotations}}\label{sec:proof}

The proofs of the theorems rely on the next lemma.

\begin{lemma}\label{lem:gain_lose}Let $\cU \subset \Sone \times \Re$ be an open set and $R_2 > R_1 >0$ be such that a connected component of $E_{[R_1,R_2]}$ is contained in $\cU$. Let $T : \cU \to \cV \eqdef \mathopen]\underline{r},\overline{r}[ \times \Sone$, $(q_u,p_u) \mapsto (r,\theta)$, be a diffeomorphism such that 
\begin{equation}\label{eq:T_property}
\begin{aligned}
&\big(\forall R \in [R_1,R_2]\big) \big(\exists s \in \mathopen]\underline{r},\overline{r}\mathclose[\big) \\
& T(E_R) = \{(r,\theta) \in \cV: r = s \},
\end{aligned}
\end{equation}
and the map $R \mapsto s$ is monotonically increasing.
Suppose that $T: \cU \to \cV$ transforms system~\eqref{eqn:acrobot-constrained-dynamics} into
\begin{equation}
\label{eq:r_theta_system}
\begin{aligned}
\dot r &= f_r(r,\theta,I) \\
\dot \theta &= f_\theta(r,\theta,I),
\end{aligned}
\end{equation}
where $f_r$, $f_\theta$ and $g \eqdef f_r / f_\theta$ satisfy, for every $(r,\theta) \in \cU$,
\begin{subequations}\label{eq:lemma_assumptions}
\begin{align}
& f_r(r,\theta,0) = 0, \ f_\theta(r,\theta,0) >0, \label{eq:lemma_assumptions:1}\\
& \partial_r g(r,\theta,0) =0, \label{eq:lemma:assumptions:2}\\
& \int_0^{2\pi} \partial_I g(r,\theta,I)\big|_{I=0} d \theta >0. \label{eq:lemma:assumptions:3} 
\end{align}
\end{subequations}
Then there exists $I^\star>0$ such that for each $I \in \mathopen]0,I^\star]$ (respectively, $I \in [-I^\star,0\mathclose[$) and every $x_0 \in E_{[R_1,R_2]} \cap \cU$, the solution $x(t,x_0)$ enjoys the following properties:

\begin{enumerate}[(a)]
\item There exists $\bar t \geq 0$ such that $x(\bar t,x_0) \in E_{R_2}$ (respectively, $E_{R_1}$).

\item If $R_2 < \bar R$, $x(t,x_0)$ is an oscillation gaining (respectively, losing) energy over $[0,\bar t]$. If $R_1 > \bar R$, $x(t,x_0)$ is a rotation gaining (respectively, losing) energy over $[0,\bar t]$.
\end{enumerate}
\end{lemma}

\begin{proof}
By property~\eqref{eq:T_property}, there exist $r_1,r_2 \in \mathopen]\underline{r},\overline{r}[$ such that $T(E_{[R_1,R_2]}) = [r_1,r_2]\times \Sone$. Moreover, $T(E_{R_2}) = \{r_2 \} \times \Sone$. Define $K \eqdef [r_1,r_2]\times \Sone$, a compact subset of $\cV$, and 
\begin{equation}\label{eq:b_defn}
b(r,\theta) \eqdef \partial_I g(r,\theta,I)\big|_{I=0}.
\end{equation}
By properties~\eqref{eq:lemma_assumptions:1} and~\eqref{eq:lemma:assumptions:3}
and since $f_r, f_\theta$ are continuous and $K$ is compact, there exist
$I_1,\varepsilon_1,\varepsilon_2>0$ such that for each $(r,\theta) \in K$ and $I
\in [-I_1,I_1]$,
\begin{subequations}
\label{eq:properties:new_coords}
\begin{align}
&f_\theta(r,\theta,I) > \varepsilon_1,\label{eq:properties:new_coords:1}\\
& \int_0^{2\pi} b(r,\theta) d \theta > \varepsilon_2.\label{eq:properties:new_coords:2}
\end{align}
\end{subequations}
Property~\eqref{eq:properties:new_coords:1} implies that for $I \in [-I_1,I_1]$
the second component of any solution $(r(t),\theta(t))$ contained in $K$ is a
strictly monotonic function of $t$, and therefore $r(t)$  can be expressed as
$r(t) = \hat r(\theta(t))$, where $\hat r(\theta)$ is a solution of
    \begin{equation}\label{eqn:rhat-ode}
            \frac{d \hat r}{d \theta} = g(\hat r,\theta,I) = \frac{f_r(\hat r,\theta,I)}{f_\theta(\hat r,\theta,I)}. 
    \end{equation}
Now let \(\hat{r}(\theta,r_0,I)\) be the solution of \eqref{eqn:rhat-ode} with initial condition $r(0)=r_0$ and $I\in [-I_1,I_1]$. The Taylor expansion of $\hat{r}(\theta,r_0,I)$ around \(I = 0\) is
\begin{equation}\label{eq:Taylor_exp}
\begin{aligned}
\hat{r}(\theta,r_0,I) &= \hat{r}(\theta,r_0,0) + I\hat{r}_1(\theta,r_0) 
    + \Delta(r_0,\theta,I) \\
&=r_0 + I\hat{r}_1(\theta,r_0) + \Delta(r_0,\theta,I),
\end{aligned}
\end{equation}
where the function $\Delta(r,\theta,I)$ is continuous and $I \mapsto
\Delta(r,\theta,I)$ is $O(I^2)$.
The second identity follows from the fact (due to~\eqref{eq:lemma_assumptions:1}) that \(g(\hat{r},\theta,0) \equiv 0\). Following~\cite[Chapter 10]{khalil_nonlinear}, the function \(\hat{r}_1(\theta,r_0)\) is the solution to the linear time-varying ODE
    \begin{equation}\label{eqn:r1-hat-prime}
\begin{aligned}         
\frac{d \hat{r}_1}{d \theta}&=
        \partial_{r} g(r_0,\theta,0)\hat{r}_1 +
        b(r_0,\theta) \\
&=b(r_0,\theta),
\end{aligned}
     \end{equation}
with initial condition $\hat r_1(0)=0$. The second identity in~\eqref{eqn:r1-hat-prime} follows from~\eqref{eq:lemma:assumptions:2}. We can solve \eqref{eqn:r1-hat-prime} by quadrature, giving
\begin{equation}
\label{eq:ronehat_solution}
\hat{r}_1(\theta,r_0) = \int_0^\theta b(r_0,\theta)d\theta.
\end{equation}
Now consider the Poincar\'e section $\cP = \{(r,\theta) \in \cV: \theta =0\}$. Using~\eqref{eq:Taylor_exp},
and~\eqref{eq:ronehat_solution}, the Poincar\'e map on $\cP$ is $P: \cP \to
\Re_{>0}$, 
\begin{equation}\label{eq:Poincare_map:expansion}
r \mapsto r + I \int_0^{2\pi} b(r,\theta) d \theta + \Delta(r,2 \pi,I).
\end{equation}
Recall that $(r,I) \mapsto \Delta(r,2 \pi,I)$ is continuous and $O(I^2)$ for each fixed $r$. Since the set $[r_1,r_2] \times [-I_1,I_1]$ is compact, there exists $\kappa>0$ such that $| \Delta(r,2\pi,I)| \leq \kappa I^2$ for all  $(r,I) \in [r_1,r_2] \times [-I_1,I_1]$. 
Using this inequality and property~\eqref{eq:properties:new_coords:2} in~\eqref{eq:Poincare_map:expansion}, we get
\[
\begin{aligned}
\big(\forall I \in \mathopen]0,I_1]\big) & \ P(r) > r + I \varepsilon_2 - \kappa I^2 \\
\big(\forall I \in [-I_1,0\mathclose[\big) & \ P(r) < r - |I| \varepsilon_2 + \kappa I^2,
\end{aligned}
\]
for all $r \in [r_1,r_2]$. Pick $I^\star \in \mathopen]0, \min\{I_1, \varepsilon_2/\kappa\} \mathclose[$ and let $\gamma \eqdef |I|\varepsilon_2 - \kappa I^2>0$. Then, 
\begin{subequations}\label{eq:poincare_inequality}
\begin{align}
(\forall r \in [r_1,r_2]) (\forall I \in \mathopen]0,I^\star]) & \ P(r) > r + \gamma \label{eq:poincare_inequality:1}\\
(\forall r \in [r_1,r_2]) (\forall I \in [-I^\star,0\mathclose[) & \ P(r) < r - \gamma. \label{eq:poincare_inequality:2}
\end{align}
\end{subequations}
Returning to the solution $\hat r(\theta,r_0,I)$ of~\eqref{eqn:rhat-ode}, property~\eqref{eq:poincare_inequality:1} implies that for each $(r_0,I) \in [r_1,r_2] \times \mathopen]0,I^\star]$, there exists an integer $\bar k \geq 0$ such that the sequence $\{r(2 \pi k, r_0,I)\}_{k =1,\ldots,\bar k}$ is monotonically increasing and $r(2 \pi \bar k,r_0,I) \geq r_2$. Vice versa, property~\eqref{eq:poincare_inequality:2} implies that for each $(r_0,I) \in [r_1,r_2] \times [-I^\star,0\mathclose[$, there exists an integer $\bar k \geq 0$ such that the sequence $\{r(2 \pi k, r_0,I)\}_{k =1,\ldots,\bar k}$ is monotonically decreasing and $r(2 \pi \bar k,r_0,I) \leq r_1$.

Now we return to system~\eqref{eq:r_theta_system} with state $(r,\theta)$, and let $(r(t),\theta(t))$ be a solution initialized in $K$, with $0< |I| \leq I^\star$. 
For $I>0$ (respectively, $I<0$), let $\bar t \geq 0$ be the first time such that $( r(\bar t), \theta(\bar t)) \in \{r_2\} \times \Sone$ (respectively, $( r(\bar t), \theta(\bar t)) \in \{r_1\} \times \Sone$). If the solution never intersects $\{r_2\} \times \Sone$, we set $\bar t = \infty$.

By property~\eqref{eq:properties:new_coords:1}, $\dot \theta > \varepsilon_1>0$, implying that the intersection of $( r([0,\bar t]),\theta([0,\bar t]) )$ with $\cP$ is a discrete set, and property~\eqref{eq:poincare_inequality} implies that the intersection points in $( r([0,\bar t]),\theta([0,\bar t]) ) \cap \cP$ form a monotonically increasing (respectively, monotonically decreasing) sequence crossing the boundary $\{r_2\} \times \Sone$ (respectively, $\{r_1\} \times \Sone$) in finite time. This in turn implies that $\bar t< \infty$. 

Returning to system~\eqref{eqn:acrobot-constrained-dynamics} in original coordinates, we have shown that for each $x_0 \in E_{[R_1,R_2]}\cap \cU$ and $0< |I| \leq I^\star$, there exists $\bar t\geq 0$ such that the solution $x(t,x_0)$ enjoys these properties:
\begin{itemize}
\item For $I>0$, $x(\bar t,x_0) \in E_{R_2}$, and for $I<0$, $x(\bar t,x_0) \in E_{R_1}$.
\item For $I>0$, $x(t,x_0)$ is an oscillation or a rotation gaining energy over $[0, \bar t]$, depending on the values of $R_1$ and $R_2$. For $I<0$, $x(t,x_0)$ is an oscillation or a rotation losing energy over $[0, \bar t]$.
\end{itemize}
This concludes the proof of Lemma~\ref{lem:gain_lose}.
\end{proof}

\subsection{Proof of Theorem~\ref{thm:acrobot-oscillations}}

We prove the theorem by producing a coordinate transformation $T : \cU \to \cV$ meeting the assumptions of Lemma~\ref{lem:gain_lose}. The transformation is adapted from~\cite{dynamic_vhcs_stabilize_closed_orbits}.

Define $\cU \eqdef E_{<\bar R} \setminus \{0\}$ and $\cV \eqdef
\mathopen]0,\pi\mathclose[ \, \times \Sone$, and consider the coordinate
transformation\footnote{Despite the presence of a $\text{sgn}(\cdot)$ function,
this coordinate transformation $T$ is actually smooth.} $T: \cU \to \cV$
defined as
\begin{equation}\label{eq:T1}
\begin{aligned}
T: & \, (q_u,p_u) \mapsto (r,\theta)  \\
r &= \arccos\left(\cos(q_u) - \frac{p_u^2}{30m^2gl^3}\right) \\
\theta &=\left.\arctan_2\left(-\sign{p_u}\sqrt{1-\frac{q_u^2}{r^2}},\frac{q_u}{r}\right)
    \right|_{r = \rone(q_u,p_u)},
\end{aligned}
\end{equation}
whose inverse is
\[
\begin{aligned}
T^{-1}: & \, (r,\theta) \mapsto (q_u,p_u)\\
q_u &= r \cos(\theta) \\
p_u &=-\sign{\sin(\theta)} 
        \sqrt{30m^2gl^3\cos(r\cos(\theta) - \cos(r))}.
\end{aligned}
\]
It is easily seen that for each $R \in \mathopen ]0,\bar R\mathclose[$, 
$T(E_R) = \{(r,\theta)\in \cV : r = \arccos(1 - R/(3 m g l))\}$, and the function
$R \mapsto \arccos(1 - R/(3 m g l))$ is monotonically increasing. The
diffeomorphism
$T: \cU \to \cV$ therefore satisfies the first requirement of
Lemma~\ref{lem:gain_lose} for any choice of $0< R_1 < R_2 < \bar R$.

In $(r,\theta)$ coordinates, $T$ maps
system~\eqref{eqn:acrobot-constrained-dynamics} into a system of the
form~\eqref{eq:r_theta_system}, where is straightforward to show that
\(f_r(r,\theta,0) \equiv 0\)
and
\begin{equation} \label{eqn:theta-dot-nom-osc}
    f_\theta(r,\theta,0) = \sqrt{\frac{6g}{5l}} 
        \sqrt{\frac{\cos(r c_\theta) - c_r}
            {r^2 s_\theta^2}}.
\end{equation}
In the above we used the notation \(c_\theta \eqdef \cos(\theta)\), \(s_\theta \eqdef \sin(\theta)\), \(c_r \eqdef \cos(r)\), and \(s_r \eqdef \sin(r)\). We note that~\eqref{eqn:theta-dot-nom-osc} has removable singularities at \(\theta \in \{0,\pi\}\), since taking the limits as \(\theta\to 0,\pi\) gives
\begin{equation}\label{eqn:theta-dot-limit}
    \lim\limits_{\theta \to 0}f_\theta(r,\theta,0)
    = \lim\limits_{\theta \to \pi} f_\theta(r,\theta,0)
    = \sqrt{\frac{6g s_r}{10lr}}
    , 
\end{equation}
and these limits are smooth and well-defined for all \(r \in \, ]0,\pi[\). From
\eqref{eqn:theta-dot-nom-osc}--\eqref{eqn:theta-dot-limit}, one can verify that
\(f_\theta(r,\theta,0) > 0\) for all $(r,\theta) \in \cV$, and thus
assumption~\eqref{eq:lemma_assumptions:1} holds. Letting \(g= f_r/f_\theta\),
one can also verify that \(\partial_r g(r,\theta,0) \equiv
0\), and assumption~\eqref{eq:lemma:assumptions:2} holds. Symbolic computations
reveal that \(\partial_I g(r,\theta,I) \big|_{I=0} = L a(r,\theta)\), where
\begin{align*}
    L &\eqdef \frac{\bar{q}_a \sqrt{30m^2g l^3}}{15}
    , \\
    a(r,\theta) &\eqdef \frac{
        r |s_\theta| \left(
        5 c_{r} \cos(r c_\theta) - 8 \cos(r c_\theta)^2 + 3
    \right)
    }{
    s_{r}\sqrt{\cos(r c_\theta) - c_{r}}
    }
    .
\end{align*}
\begin{figure}[htb]
    \centering
    \includegraphics[width=0.8\linewidth]{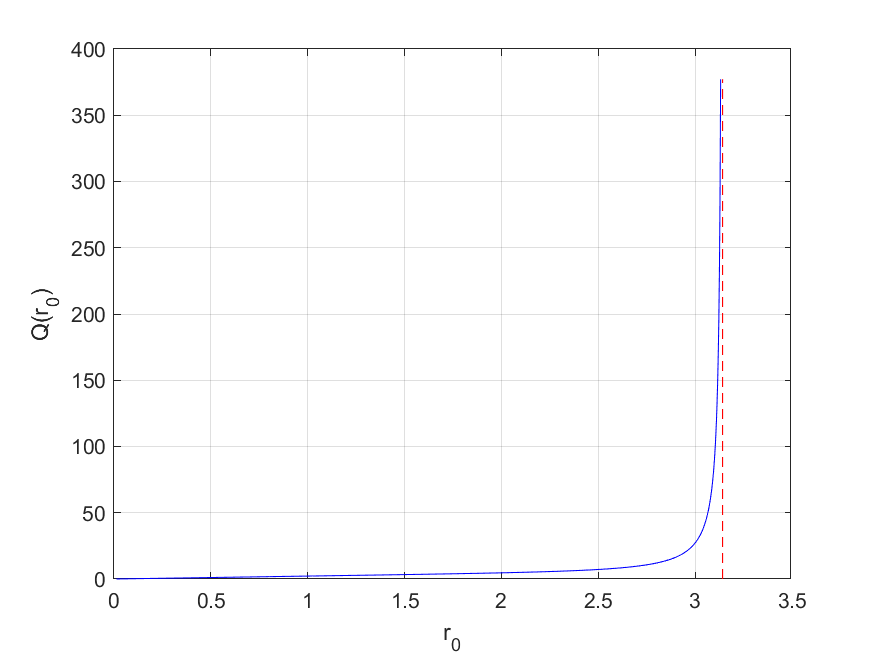}
    \caption{The graph of \(r \mapsto \int_0^{2\pi} a(r,\theta) d\theta\).}
    \label{fig:acrobot-Q}
\end{figure}
Notice that \(L\) is a positive constant which depends only on \(m\), \(g\), \(l\), and \(\bar{q}_a\), while \(a(r,\theta)\) is adimensional. We have
\[
\int_0^{2\pi} \partial_I g(r,\theta,I) \big|_{I=0} d \theta =L \int_0^{2\pi} a(r,\theta) d\theta.
\]
The graph of \(\int_0^{2\pi} a(r,\theta) d\theta\), shown in
Figure \ref{fig:acrobot-Q}, indicates that for each $r \in \mathopen] 0,\pi\mathclose[$, the function $r \mapsto \int_0^{2\pi} a(r,\theta) d\theta$ is positive, and assumption~\eqref{eq:lemma:assumptions:3} holds. Thus the hypotheses of Lemma~\ref{lem:gain_lose} are satisfied, proving the theorem.
\hfill\QED

\subsection{Proof of Theorem~\ref{thm:acrobot-rotations}}

We use again Lemma~\ref{lem:gain_lose}, this time producing two coordinate transformations, adapted from~\cite{dynamic_vhcs_stabilize_closed_orbits}, whose domains are the two connected components of the set $E_{>\bar R}$.

Define $\cU^+ \eqdef \big(E_{>\bar R}\big) \cap \{ p_u >0\}$, and $\cU^- \eqdef \big(E_{>\bar R}\big) \cap \{ p_u <0\}$. Let $\cV \eqdef \mathopen] (10 m l^2 \bar R)^{1/2}, \infty \mathclose[ \times \Sone$, and consider the coordinate transformations $T^+ : \cU^+ \to \cV$ and $T^- : \cU^- \to\cV$ defined as
\begin{equation}\label{eq:T2}
\begin{aligned}
T^\pm: & \, (q_u,p_u) \mapsto (r,\theta) \\
r &= \sqrt{p_u^2 + 30m^2gl^3(1-c_u)} \\
\theta &= \pm q_u,
\end{aligned}
\end{equation}
whose inverses are
\[
\begin{aligned}
(T^\pm)^{-1}: & \, (r,\theta) \mapsto (q_u,p_u)\\
q_u &= \pm \theta \\
p_u &= \pm \sqrt{r^2-30m^2gl^3(1-c_\theta)}.
\end{aligned}
\]
For each $R> \bar R$, it is easily seen that $T^\pm(E_R) = \{(r,\theta) \in \cV : r = (10 m l^2 R)^{1/2}\}$, and the function $R \mapsto (10 m l^2 R)^{1/2}$ is monotonically increasing. Both $T^+$ and $T^-$, therefore, satisfy the first requirement of Lemma~\ref{lem:gain_lose} for any $\bar R < R_1 < R_2$. 

In $(r,\theta)$ coordinates, both $T^+$ and $T^-$ map
system~\eqref{eqn:acrobot-constrained-dynamics} to a system of the
form~\eqref{eq:r_theta_system}, where one can verify that
\(f_r(r,\theta,0) \equiv 0\) and
\[
 f_\theta(r,\theta,0) = \frac{1}{5ml^2} \sqrt{r^2 - 30m^2gl^3(1-c_\theta)}
,
\]
which is positive so assumption~\eqref{eq:lemma_assumptions:1} holds. Letting $g
= f_r / f _\theta$, one can verify that \(\partial_r g (r,\theta,0)
\equiv 0\), and assumption~\eqref{eq:lemma:assumptions:2} holds. The function
\(b(r,\theta) = \partial_I g(r,\theta,0)\) is given in the statement
of Theorem~\ref{thm:acrobot-rotations}, and by assumption we have that
condition~\eqref{eq:lemma:assumptions:3} of Lemma~\ref{lem:gain_lose} holds.
Theorem~\ref{thm:acrobot-rotations} now follows from  Lemma~\ref{lem:gain_lose}
and the fact that $\cU^+ \cup \cU^- = E_{>\bar R}$.
\hfill \QED

\section{Conclusion}\label{sec:conclusion}

In this article we applied the framework of virtual nonholonomic constraints to
the acrobot, and designed a \vnhc emulating giant motion from gymnastics. Our
theoretical analysis applies to a simplified acrobot with point-mass limbs of
equal length and mass, but the simulation and experimental results show the
validity of the approach for a real acrobot.  The main theoretical results of
the paper, Theorems~\ref{thm:acrobot-oscillations}
and~\ref{thm:acrobot-rotations}, assert that the proposed \vnhc makes the
dynamics of the acrobot on the constrained manifold gain or lose energy. There
are two theoretical points that were unexplored in this paper. First, our
analysis only considers initial conditions on the constraint manifold. For
initial conditions off the manifold, there will be a transient whose effects we
have not analyzed theoretically. Second, as pointed out in
Remark~\ref{rem:thm:caveat}, we have shown energy gain/loss of the constrained
dynamics on compact subsets $E_{[R_1,R_2]}$ of $E_{<\bar R}$ and $E_{>\bar R}$,
so both oscillations and rotations of the acrobot will gain/lose energy. We have
not given theoretical guarantees that oscillations gaining energy eventually
turn into rotations gaining energy.
Both our simulations and the physical experiments suggest that our theoretical
predictions continue to hold in the face of transients in enforcing the
constraint, and that oscillations gaining energy will turn into rotations
gaining energy.

\bibliography{bib}
\end{document}